\documentclass[10pt]{article}
\usepackage[preprint]{tmlr}

\usepackage{amsmath,amssymb,amsthm,mathtools}
\usepackage{bm}
\usepackage{aliascnt}
\usepackage{booktabs}
\usepackage{graphicx}
\usepackage{hyperref}
\hypersetup{
  pdftitle={Kernel of Partition Paths: A Unified Representation for Tree Ensembles},
  pdfauthor={Nicolas Mahler},
  pdfsubject={Geometric representation for tree ensembles unifying prediction, attribution, robustness, and generalisation bounds.},
  pdfkeywords={tree ensembles, decision trees, kernel methods, Rademacher complexity, generalisation bounds, robustness, feature attribution, path metric, cross-fitting}
}
\usepackage{url}
\usepackage[nameinlink,noabbrev]{cleveref}

\usepackage{amsmath,amsfonts,bm}

\def\eqref#1{equation~\ref{#1}}

\def\1{\bm{1}}

\DeclareMathAlphabet{\mathsfit}{\encodingdefault}{\sfdefault}{m}{sl}
\SetMathAlphabet{\mathsfit}{bold}{\encodingdefault}{\sfdefault}{bx}{n}

\newcommand{\E}{\mathbb{E}}

\newcommand{\R}{\mathbb{R}}

\newcommand{\Var}{\mathrm{Var}}

\DeclareMathOperator{\sign}{sign}
\DeclareMathOperator{\Tr}{Tr}

\newcommand{\Xspace}{\mathcal X}
\newcommand{\Yspace}{\mathcal Y}

\newcommand{\Fcal}{\mathcal F}
\newcommand{\Hcal}{\mathcal H}

\renewcommand{\R}{\mathbb R}

\renewcommand{\E}{\mathbb E}
\newcommand{\Prob}{\mathbb P}
\newcommand{\one}{\mathbf{1}}

\newcommand{\parnode}{\operatorname{par}}

\newcommand{\Leaves}{V^{\mathrm{leaf}}}
\newcommand{\Internal}{V^{\mathrm{int}}}
\newcommand{\lca}{\operatorname{lca}}

\newcommand{\nodew}{a}
\newcommand{\edgew}{\omega}
\newcommand{\phit}{\phi_t}
\newcommand{\phiT}{\phi_T}

\newcommand{\tphiT}{\widetilde\phi_T}
\newcommand{\dt}{d_t}
\newcommand{\dT}{d_T}
\newcommand{\deltaT}{\delta_T}
\newcommand{\ST}{S_T}

\newcommand{\GT}{G_T}
\newcommand{\GK}{G_K}
\newcommand{\Ginf}{G_\infty}
\renewcommand{\Tr}{\operatorname{Tr}}

\newcommand{\op}{\mathrm{op}}
\newcommand{\deff}{d_{\mathrm{eff}}}
\newcommand{\Deltak}{\Delta_K}
\newcommand{\mueig}{\mu}

\renewcommand{\Var}{\operatorname{Var}}

\newcommand{\Risk}{\mathcal R}

\newcommand{\lreg}{\lambda_{\mathrm{reg}}}

\newcommand{\norm}[1]{\left\lVert #1 \right\rVert}
\newcommand{\abs}[1]{\left\lvert #1 \right\rvert}
\newcommand{\inner}[2]{\left\langle #1,#2 \right\rangle}
\newcommand{\set}[1]{\left\{#1\right\}}
\newcommand{\given}{\,\middle|\,}

\newcommand{\Rcal}{\mathfrak R}

\newcommand{\Phinode}{\tphiT}

\newcommand{\Knode}{\GT}

\newcommand{\Riskhat}{\widehat{\Risk}}

\theoremstyle{plain}
\newtheorem{theorem}{Theorem}[section]

\newaliascnt{proposition}{theorem}
\newtheorem{proposition}[proposition]{Proposition}
\aliascntresetthe{proposition}

\newaliascnt{lemma}{theorem}
\newtheorem{lemma}[lemma]{Lemma}
\aliascntresetthe{lemma}

\newaliascnt{corollary}{theorem}
\newtheorem{corollary}[corollary]{Corollary}
\aliascntresetthe{corollary}

\theoremstyle{definition}
\newaliascnt{definition}{theorem}
\newtheorem{definition}[definition]{Definition}
\aliascntresetthe{definition}

\newaliascnt{assumption}{theorem}
\newtheorem{assumption}[assumption]{Assumption}
\aliascntresetthe{assumption}

\theoremstyle{remark}
\newaliascnt{remark}{theorem}
\newtheorem{remark}[remark]{Remark}
\aliascntresetthe{remark}

\crefname{theorem}{Theorem}{Theorems}
\Crefname{theorem}{Theorem}{Theorems}
\crefname{proposition}{Proposition}{Propositions}
\Crefname{proposition}{Proposition}{Propositions}
\crefname{lemma}{Lemma}{Lemmas}
\Crefname{lemma}{Lemma}{Lemmas}
\crefname{corollary}{Corollary}{Corollaries}
\Crefname{corollary}{Corollary}{Corollaries}
\crefname{definition}{Definition}{Definitions}
\Crefname{definition}{Definition}{Definitions}
\crefname{assumption}{Assumption}{Assumptions}
\Crefname{assumption}{Assumption}{Assumptions}
\crefname{remark}{Remark}{Remarks}
\Crefname{remark}{Remark}{Remarks}

\title{Kernel of Partition Paths:\\A Unified Representation for Tree Ensembles}

\author{\name Nicolas Mahler \email nicolas.mahler@datapred.com \\
        \addr Datapred SAS \\
        23 rue Mirabeau \\
        94300 Vincennes, France \\
        \\
        Datapred SA \\
        EPFL Innovation Park -- Bâtiment A \\
        1015 Lausanne, Switzerland}

\begin{document}

\maketitle

\begin{abstract}
A recent line of work has reframed individual decision trees as linear
models on engineered features associated with their splits, opening
routes for oracle inequalities and feature-importance reinterpretation,
but leaving open the question of what unified geometric object a forest
induces when one indexes its feature map by nodes rather than
by splits. The present paper studies that object. KPP indexes the
feature map by the nodes of the forest, weighted by a path
metric that turns each coordinate into a component of a
squared-Euclidean path-isometric embedding. KPP unifies four pillars
under a single node-indexed representation whose Gram is non-diagonal and carries a metric: prediction,
exact additive attribution, deterministic Lipschitz robust radius in
the KPP metric, and uniform Rademacher risk bounds for regression and
classification under fixed, honest, or cross-fit conditioning. All
probabilistic guarantees are conditional on the representation and are
stated under three explicit conditioning regimes; the robust-radius
guarantee is deterministic in the KPP metric rather than in a norm on
the raw input. Conjectured fast-rate refinements for both regression
and classification are stated as open problems and are not claimed as
theorems.
\end{abstract}

\section{Introduction}
\label{sec:intro}
Random forests \citep{Breiman2001RandomForests} have remained one of the
standard predictive tools for tabular data, and have also served as a
substrate for the theoretical study of ensemble methods and for the
design of model-level interpretation tools. A recent line of work has
reframed individual CART trees as linear models on engineered features
associated with their splits \citep{KlusowskiTian2024LargeScale}, and
has used this reframing to extend forests with explicit regularisation
and generalised linear link functions \citep{Agarwal2025MDIPlus}. This
line of work shows that the partitioning power of forests and the
analytical tractability of linear regression can be combined within a
single estimator, and that the resulting linear view opens routes for
oracle inequalities, feature-importance reinterpretation, and
inductive-bias analysis. The present paper sits within this broad
direction and asks a related but distinct question: what geometric
object does a forest induce when one indexes its feature map not by
splits but by nodes, and what guarantees does that object
support.

The recent literature on linear and geometric views of forests splits
along two structurally distinct directions, surveyed in
\cref{sec:related-work}. One direction, descending from rule ensembles
\citep{FriedmanPopescu2008RuleEnsembles} and refined in recent work on
linear views of CART, indexes features by splits and obtains a diagonal
feature Gram through the orthogonality of per-split decision stumps.
Another direction, descending from forest proximities and from the view
of forests as adaptive kernels \citep{Scornet2016RFKernel}, indexes
similarity by leaf co-membership and exposes a kernel but no explicit
metric over the internal structure of the trees. The first direction
supports linear-model machinery on top of the forest at the price of a
representation whose geometry is degenerate (a diagonal feature Gram
cannot encode path distance). The second direction supports kernel machinery
at the price of losing the additive split-level decomposition that the
stump-based view affords. These two directions, taken individually,
operate within narrower scopes than a single representation that
simultaneously supports prediction, attribution, robustness, and
generalisation.

KPP indexes its feature map by the nodes of the forest, with
each coordinate weighted by a path metric on the tree
(\cref{sec:setup-representation}). The resulting embedding is
squared-Euclidean and path-isometric: distances in the embedded space
recover the weighted path distance through the trees up to a
normalisation factor (\cref{thm:kpp-isometry-tree}). Unlike the
stump-based designs, whose orthogonal per-split features give a diagonal
feature Gram $\Phi^\top \Phi$ and a ridge fit that decouples across
coordinates, and unlike the leaf-based kernels that expose similarity
without a metric on the internal tree structure, KPP uses a
non-orthogonal node system whose coordinates overlap along
ancestor--descendant chains: its feature Gram $\Phi^\top \Phi$ is
non-diagonal, and its sample Gram $\Phi \Phi^\top$ carries the
path-metric structure directly, being non-diagonal whenever two points
share a node of positive weight (\cref{cor:gram-psd-nondiagonal}). KPP
unifies four pillars under a single node-indexed representation. Namely:
prediction
(\cref{sec:setup-risk}), exact additive attribution
(\cref{sec:attribution}), deterministic Lipschitz robust radius in the
KPP metric (\cref{sec:robust-radius}), and uniform Rademacher risk
bounds under fixed, honest, or cross-fit conditioning
(\cref{sec:rademacher,sec:proof-boundary}). Each of these pillars is
established under explicit conditioning hypotheses on the
representation.

The probabilistic guarantees in this paper are conditional on the KPP
representation and are stated under three explicit conditioning
regimes---fixed, honest, and cross-fit---catalogued as PB-01, PB-02,
and PB-03 in \cref{sec:proof-boundary} and discussed in
\cref{sec:radical-honesty}. The robust-radius guarantee is
deterministic in the KPP metric $\deltaT$ rather than in an $\ell_p$
perturbation of the raw input (\cref{sec:robust-radius}); this is a
guarantee in the geometry of the forest, not in the geometry of the
covariate space. Conjectured fast-rate refinements for both regression
(FR-R-01) and classification (FR-C-01) are stated as open problems in
\cref{sec:discussion} and are not claimed as theorems in the present
paper.

The paper is organised as follows. \Cref{sec:related-work} situates the
KPP representation within prior work on linear and geometric views of
forests. \Cref{sec:setup} introduces the representation, the
normalisation, and the risk framework. \Cref{sec:kpp-construction}
establishes the squared-Euclidean path-isometric embedding, the
non-diagonal Gram matrix, and the finite-forest concentration of the
embedded geometry. \Cref{sec:attribution} derives the exact additive
attribution that follows from the node-indexed embedding.
\Cref{sec:robust-radius} establishes the deterministic Lipschitz robust
radius in the KPP metric. \Cref{sec:rademacher} proves the uniform
Rademacher risk bounds for regression and classification, and
\cref{sec:radical-honesty} sets out the asymmetry between regression
and classification guarantees under unconstrained ERM.
\Cref{sec:experiments} reports the empirical behaviour of KPP on five
tabular benchmarks and on a methodologically motivating dataset.
\Cref{sec:discussion} states the open problems, including the
fast-rate conjectures and the directions left unexplored by the proof
boundaries of the present work.

\section{Related work}
\label{sec:related-work}
We organise prior work around the linear and geometric views of
decision trees and forests, which are the main reference points for
the present construction. The discussion is factual: we describe
what each method computes and the regime under which its analysis
applies.

\paragraph{Linear views of decision trees and forests.}
\citet{KlusowskiTian2024LargeScale} establish that the predictions
of a regression or classification tree obtained by CART or C4.5
methodology are equivalent to the OLS fit of the response on the
local stump feature map indexed by the internal nodes of the tree.
This identity is the theoretical bedrock of linear views: it makes
the per-tree predictor a linear model in an explicit, tree-dependent
feature space.

Building on the linear-view identity, several methods regularise
tree-derived feature maps or fit linear models at the leaf level.
\citet{FriedmanPopescu2008RuleEnsembles} fit a LASSO model on
rule-indicator features extracted from a tree ensemble (the RuleFit
construction), combining the rules with the original linear
variables. \citet{Agarwal2022HierarchicalShrinkage} show that the
hierarchical-shrinkage post-processing of a fitted tree is
equivalent to a ridge regression on the internal-node stump
features, providing a linear-model reading of a shrinkage scheme
that was originally introduced for accuracy and interpretability.
\citet{Raymaekers2025RaFFLE} use PILOT trees with linear leaf
models as base learners in a random-forest ensemble, extending the
linear view from internal-node features to leaf-level linear fits.

A parallel lineage extends the mean-decrease-in-impurity (MDI)
feature-importance scheme into local and regularised forms.
\citet{Sutera2021LocalMDI} introduce a Local MDI framework that
decomposes the impurity-decrease score point-by-point and connects
local MDI scores to Shapley values \citep{LundbergLee2017SHAP,LundbergErionChen2020TreeSHAP} under specific conditions.
\citet{Agarwal2025MDIPlus} (MDI+ / RF+) fit a regularised
generalised linear model on the internal-node stump features of a
forest, with a sample-splitting protocol that separates the data
used to grow the trees from the data used to fit the GLM
coefficients; the reported guarantees are conditional on this
honest split. \citet{Liang2025LMDIPlus} (LMDI+) extend the MDI+
construction to point-wise local feature importances by fitting a
regularised GLM at each query point, on the same internal-node
stump feature space.

The present work occupies a different route within the same broad
direction. KPP indexes its feature map by the nodes of the
forest, weighted by a path metric that turns each node
coordinate into a component of a squared-Euclidean
path-isometric embedding (\cref{sec:setup-representation,thm:kpp-isometry-tree}).
Unlike the stump-based designs, whose orthogonal per-split features
produce a diagonal feature Gram $\Phi^\top \Phi$ and a coordinate-wise
decoupled ridge fit, KPP uses a non-orthogonal node system, so its
sample Gram $\Phi \Phi^\top$ is non-diagonal whenever two points share a
node of positive weight (\cref{cor:gram-psd-nondiagonal}) and carries
the path-metric structure directly. This
single representation supports four geometrically-coherent pillars:
prediction (\cref{sec:setup-risk}), exact additive attribution
(\cref{sec:attribution}), deterministic Lipschitz robust radius in
the KPP metric (\cref{sec:robust-radius}), and uniform Rademacher
risk bounds under fixed, honest, or cross-fit conditioning
(\cref{sec:rademacher,sec:proof-boundary}). It is a structurally
distinct construction within the same broad direction. On the attribution
axis, the KPP decomposition of \cref{sec:attribution} is an exact algebraic
accounting of the fitted predictor by marginal split, distinct in kind from the
game-theoretic attributions of \citet{Sutera2021LocalMDI} and the SHAP family
\citep{LundbergLee2017SHAP, LundbergErionChen2020TreeSHAP}: it satisfies no
Shapley axioms and is conditional on the split path
(\cref{rem:attribution-not-axiomatic}), and the variable-level
decomposition rests on a non-unique, minimum-norm choice over an
overcomplete edge design (\cref{rem:attribution-identifiability}).

A separate family of forest-derived kernels takes the leaves of the
trees as the unit of geometry: each pair of points is assigned a
proximity proportional to the number of trees in which they share a
leaf. The proximity construction originates with
\citet{ShiHorvath2006UnsupervisedRF}, who used it for unsupervised
random-forest analysis and multidimensional-scaling visualisation.
\citet{DaviesGhahramani2014RandomPartitionKernels} introduced
random-partition kernels in this spirit, and
\citet{Scornet2016RFKernel} analysed the resulting forest kernel in
a regression setting. \citet{Rhodes2023RFGAP} (RF-GAP) refine the
leaf-based proximity through a kernel-based reweighting that
preserves the forest's predictive accuracy.
\citet{Panda2018KMERF} (KMERF) construct characteristic kernels from
forest leaf occupancy for non-parametric independence testing, and
\citet{FengZhou2017AutoEncoderForest} use forest leaf occupancy as
an encoder in an autoencoder construction.
\citet{VuKaparWrightWatson2025RFAE} extend the leaf-based approach
to a forest-based autoencoder for missing-data imputation. These
proximities operate at the leaf level. The KPP path metric used here
is indexed by the nodes of the forest, with weights that
turn it into a squared-Euclidean path isometry; this is a
structurally distinct geometric object.

\paragraph{Theory of random forests.}
On the theoretical side, the random-forests literature provides
multiple frameworks for analysing forest-based estimators.
\citet{BiauScornet2016GuidedTour} provide a guided tour of the
analysis techniques developed during the first decade of post-Breiman
random-forest theory. \citet{ScornetHooker2025TheoryReview} survey
the more recent developments. \citet{MourtadaGaiffasScornet2020Mondrian}
establish minimax-optimal rates for Mondrian trees and forests under
specific generative models. The present paper uses standard
Rademacher complexity arguments in \cref{sec:rademacher}, with the
KPP-specific input being the controlled trace
$\Tr(\Knode) \le n$ of \cref{prop:norm-bound}.

\paragraph{Robustness verification of tree ensembles.}
A separate research line addresses robustness verification of tree
ensembles against perturbations of the raw input.
\citet{Chen2019RobustnessVerification} verify boosted decision trees
against $\ell_p$ perturbations of the raw input, and
\citet{AndriushchenkoHein2019ProvablyRobust} provide
provably-robust training procedures for boosted decision stumps and
trees under $\ell_\infty$ attack budgets. The robust radius
certificates of \cref{sec:robust-radius} are of a different nature:
they are deterministic in the KPP path metric $\deltaT$ and certify
the KPP linear model, not the underlying forest, against
perturbations measured in $\deltaT$, not in the raw input metric
(see also \cref{sec:discussion}).

\paragraph{Honest forests and proof-boundary conventions.}
The honest and cross-fit constructions used in
\cref{sec:proof-boundary,sec:diabetes-honest-crossfit} have their
origin in the causal-inference literature.
\citet{WagerAthey2018HonestForests} introduced honest random forests
for treatment-effect estimation, and
\citet{AtheyTibshiraniWager2019GRF} developed the generalised
random-forest framework. The present paper does not use honesty or
cross-fitting for causal inference. Following the proof-boundary
discipline of \cref{sec:proof-boundary}, we use the same data-
splitting mechanism so that the trees, the feature map, and the
Gram matrix can be treated as fixed conditionally when the ridge or
logistic estimator is analysed in \cref{sec:rademacher}. The
calibration arguments used in
\cref{sec:rademacher-classification} draw on standard
classification-calibration results
\citep{BartlettJordanMcAuliffe2006Calibration,MammenTsybakov1999SmoothDiscrimination},
which are independent of the KPP representation.

\paragraph{Embeddings, geometric forests, and surveys.}
\citet{LinialLondonRabinovich1995Geometry} established classical
metric-embedding results, including the embedding of tree metrics
into $L_1$, in the graph-theoretic geometry literature. The KPP
path metric is structurally distinct: it is a node-indexed
squared-Euclidean isometry on a re-weighted path metric, not on the
raw tree metric. A separate line of geometric extensions of random
forests addresses non-Euclidean predictors and outcomes.
\citet{Capitaine2024FrechetRF} develop Fréchet random forests for
metric-space-valued regression with non-Euclidean predictors, and
\citet{BenardNafJosse2024MMD} introduce MMD-based variable
importance for distributional random forests. These approaches
modify the loss or the prediction space; the KPP construction
modifies the feature map of a standard regression or classification
forest. Broader interpretability surveys of random-forest methods
are provided by \citet{HaddouchiBerrado2024Survey}.

\section{Setup and notation}
\label{sec:setup}

\subsection{Data}
\label{sec:setup-data}

\begin{assumption}[Data]
\label{ass:data}
The training sample $(X_i, Y_i)_{i=1}^n$ consists of i.i.d.\ draws
from an unknown distribution $P$ on $\Xspace \times \Yspace$, with
$\Xspace \subseteq \R^p$. The label space depends on the task:
$\Yspace \subseteq \R$ with $\E[Y^2] < \infty$ for regression, and $\Yspace = \set{-1, +1}$
for binary classification.
\end{assumption}

We write $\E$ for expectation under $P$, reserve $\Risk$ for the
population risk and $\Riskhat$ for its empirical counterpart on the
training sample (\cref{sec:setup-risk}).

\subsection{Forests, trees, and partition paths}
\label{sec:setup-forests}

A tree $t$ is a finite rooted binary tree with internal nodes
$\Internal$ and leaves $\Leaves$. For a node $v$, $\parnode(v)$
denotes its parent and $\lca(v, v')$ the lowest common ancestor of
$v$ and $v'$. Each internal node carries an axis-aligned split. A
\emph{forest} is a finite collection $T = (t_1, \dots, t_K)$ of trees,
possibly grown with bootstrap or feature subsampling
\citep{Breiman2001RandomForests}.

To each tree $t$ and node $v$ we attach a nonnegative \emph{node
weight} $\nodew_t(v) \ge 0$ summarising the impurity decrease or
local variance at $v$. The associated \emph{edge weight} is
$\edgew_t(v) = (\nodew_t(\parnode(v)) + \nodew_t(v))/2$ for non-root
nodes, with the convention that the edge into the root carries
weight zero. Full definitions, including the global aggregate
$\ST$, are deferred to \cref{sec:kpp-construction}.

\paragraph{Notation policy.}
The rewrite keeps different mathematical roles separated: node
weights are denoted $\nodew_t(v)$, induced edge weights
$\edgew_t(v)$, regularization parameters $\lreg$, and Gram
eigenvalues $\mueig_j$. This avoids the earlier ambiguity in which
the same symbol was used for node weights, ridge parameters, and
eigenvalues.

\subsection{KPP representation}
\label{sec:setup-representation}

The KPP construction (\cref{sec:kpp-construction}) attaches to the forest $T$ an
explicit node-indexed feature map
\[
  \Phinode : \Xspace \to \R^m,
  \qquad
  \norm{\Phinode(x)}_2 \le 1 \text{ for every } x \in \Xspace.
\]
The associated forest path distance is $\dT(x, x')$ and its
normalised counterpart $\deltaT(x, x')$. The Gram matrix of the
embedding on the training inputs is
\begin{equation}
  (\Knode)_{ij} \;:=\; \inner{\Phinode(x_i)}{\Phinode(x_j)},
  \qquad i, j \in \set{1, \dots, n}.
  \label{eq:setup-gram}
\end{equation}
The load-bearing isometry identity, namely
$\norm{\Phinode(x) - \Phinode(x')}_2^2 = \deltaT(x, x')$, is proved
in \cref{sec:kpp-construction}. The symbols $\Phinode, \dT, \deltaT, \Knode, \ST$ are
used in the rest of the paper without further definition.

\subsection{Hypothesis class, prediction, and risk}
\label{sec:setup-risk}

We consider the linear KPP hypothesis class
\begin{equation}
  \Hcal_B \;=\; \set{x \mapsto \inner{w}{\Phinode(x)} + b
                    : w \in \R^m,\, b \in \R,\, \norm{w}_2 \le B},
  \label{eq:hcal-B}
\end{equation}
parametrised by a norm budget $B \ge 0$. All Rademacher and
uniform-loss bounds in
\cref{sec:rademacher-regression,sec:rademacher-classification} are
stated on the homogeneous linear sub-class
$\Fcal_B := \Hcal_B|_{b = 0}$ introduced in
\cref{eq:fcal-B-linear}. For squared-loss regression, an intercept
in $\Hcal_B$ can be absorbed by centering both the response and the
feature columns: at the optimum the intercept profiles out as
$b^\star = \bar y - \inner{w}{\bar\phi}$, with $\bar y$ and $\bar\phi$
the empirical means of the response and of $\Phinode$, so the affine
problem is equivalent to the homogeneous one on the centered features
without changing the relevant capacity. For logistic-loss classification,
mean-centering features does not eliminate an additive intercept
$b$ from the score; we therefore state all classification bounds
on $\Fcal_B$ (i.e., with $b = 0$), matching the convention adopted
for theoretical scores throughout
\cref{sec:rademacher-classification}. Extending the bounds to the
full affine class would require an additional norm constraint on
$b$, which we do not pursue here. The reference implementation fits a
separate unpenalised intercept (the scikit-learn default), which lies
outside $\Fcal_B$; the reported classification bounds accordingly
pertain to the homogeneous predictor; the constant root coordinate of
\cref{def:tree-embedding} supplies an intercept-like term within the
norm-constrained $w$ only when the root weight is positive
($\nodew_t(r_t) > 0$). Removing the separate unpenalised intercept then
aligns the implementation with the analysed class at the cost of a
penalised bias absorbed into the norm budget, not without loss of
expressivity.

The score $g_w(x) = \inner{w}{\Phinode(x)} + b$ enters the two tasks
as follows.
\begin{itemize}
  \item \textbf{Regression.} The predictor is $\widehat f(x) = g_{\widehat w}(x)$
    and the loss is squared, $\ell(g, y) = (y - g)^2$.
  \item \textbf{Classification.} The predicted label is
    $\widehat y(x) = \sign(g_{\widehat w}(x))$, with binary labels
    $Y \in \set{-1, +1}$ and the logistic surrogate loss
    $\ell(g, y) = \log(1 + \exp(-y g))$.
\end{itemize}
The population risk and its empirical counterpart on the training
sample are
\begin{equation}
  \Risk(g) \;=\; \E\!\left[\ell(g(X), Y)\right],
  \qquad
  \Riskhat(g) \;=\; \frac{1}{n}\sum_{i=1}^{n} \ell(g(X_i), Y_i),
  \label{eq:risk-emp-risk}
\end{equation}
where the loss $\ell$ is fixed once a task is specified.

\subsection{Proof boundary: fixed, honest, and cross-fit representations}
\label{sec:proof-boundary}

All statistical statements in this paper use the deterministic KPP
geometry from \cref{sec:kpp-construction}. The proof boundary is the following pipeline:
\begin{equation}
  (X_i, Y_i)_{i=1}^n
  \;\longrightarrow\;
  (T, \nodew, \Phinode)
  \;\longrightarrow\;
  \widehat f .
  \label{eq:full-kpp-pipeline}
\end{equation}
The first arrow builds the forest, node weights, and KPP
representation. The second arrow fits the final convex learner
(ridge for regression, logistic ERM for classification). Bounds for
linear or kernel methods normally analyse only the second arrow,
after the representation is fixed. They should therefore be stated
under one of the regimes below.

\paragraph{Fixed representation.}
The forest, node weights, and features $\Phinode(x_i)$ are regarded
as fixed before the final learner is analysed. Equivalently, the
theorem is conditional on the representation. This is the cleanest
setting for Rademacher bounds, ridge identities, calibration
inequalities, and deterministic robustness certificates.

\paragraph{Honest representation.}
The sample is split into a partition fold and a fit fold. The
partition fold builds the trees and node weights. The fit fold
constructs the KPP design and fits the final ERM. Conditionally on
the partition fold, the representation on the fit fold is fixed with
respect to the labels used by the final ERM. This is the safest
route for oracle quantities such as entropy or variance gains
\citep{WagerAthey2018HonestForests, AtheyTibshiraniWager2019GRF}.

\paragraph{Cross-fit representation.}
The sample is split into folds. For each fold $q$, a representation
$\Phinode^{(-q)}$ is learned without using labels from fold $q$, and
points in fold $q$ are embedded with that representation.
Cross-fitting reduces label leakage while preserving more training
data than a single split. A theorem using cross-fitting should still
state its conditioning explicitly, because the collection of
cross-fit features can depend on labels in other folds.

\paragraph{Label-dependent representations (PB-04).}
If the same labels are used both to choose splits or node weights and
to fit the final linear model, then the feature map is
label-dependent. This does not make the algorithm invalid, but it
changes the proof problem: standard fixed-class Rademacher or ridge
arguments become diagnostics for the fitted representation, not
end-to-end guarantees for the full pipeline in
\cref{eq:full-kpp-pipeline}. An end-to-end theorem would need a
separate analysis of the representation-building step, for example
through full-pipeline stability or a uniform control over the class
of possible forests.

\paragraph{Boilerplate for theorems.}
Every statistical theorem in this paper carries one of the following
clauses:
\begin{quote}
  \emph{Conditionally on the KPP representation, the following bound
  holds for the final ERM.}
\end{quote}
\begin{quote}
  \emph{Assume the representation is learned honestly on data
  disjoint from the final fitting fold.}
\end{quote}
\begin{quote}
  \emph{Assume a cross-fit representation and condition fold-wise on
  the representations used to embed each held-out fold.}
\end{quote}
This convention prevents deterministic KPP geometry, fixed-design
learning theory, and full-pipeline guarantees from being conflated.

\section{KPP construction: embedding, metric, Gram, isometry, and finite-forest concentration}
\label{sec:kpp-construction}

This section defines the deterministic KPP geometry for a fixed tree
and then for a fixed finite forest. No statistical assumption is used
here: the tree, its partition, and its node weights are regarded as
fixed objects.

\subsection{Trees, nodes, leaves, and paths}
\label{sec:trees-paths}

\begin{definition}[Recursive binary partition]
\label{def:recursive-binary-partition}
A rooted binary tree is a triple $t = (V_t, E_t, S_t)$, where $V_t$
is the finite set of nodes, $E_t$ is the set of directed edges, and
$S_t = (S_v)_{v \in V_t}$ is a family of measurable subsets of
$\Xspace$. The root is denoted $r_t$, with $S_{r_t} = \Xspace$. If
an internal node $v$ has children $v_L, v_R$, then
\begin{equation}
  S_v = S_{v_L} \cup S_{v_R},
  \qquad
  S_{v_L} \cap S_{v_R} = \varnothing.
\end{equation}
The set of leaves is denoted $V_t^{\mathrm{leaf}}$ and the set of
internal nodes $V_t^{\mathrm{int}}$.
\end{definition}

For $x \in \Xspace$, let $\ell_t(x) \in V_t^{\mathrm{leaf}}$ be the
leaf reached by $x$. Let $\pi_t(x) \subset V_t$ be the root-to-leaf
path from $r_t$ to $\ell_t(x)$. For two points $x, x'$, let
$\sigma_t(x, x') \subset V_t$ be the unique simple path connecting
$\ell_t(x)$ and $\ell_t(x')$.

\begin{lemma}[Path decomposition]
\label{lem:path-decomposition}
Let $u = \lca(\ell_t(x), \ell_t(x'))$ be the lowest common ancestor
of the two leaves. Then
\begin{equation}
  \sigma_t(x, x') =
    \bigl(\pi_t(x) \triangle \pi_t(x')\bigr) \cup \set{u},
\end{equation}
where $\triangle$ denotes symmetric difference.
\end{lemma}

\begin{proof}
In a tree there is a unique simple path between any two nodes. The
path from $\ell_t(x)$ to $\ell_t(x')$ goes upward from $\ell_t(x)$ to
the lowest common ancestor $u$, and then downward from $u$ to
$\ell_t(x')$. The nodes strictly below $u$ on either branch are
exactly the symmetric difference of the two root-to-leaf paths, and
$u$ is the deepest common node.
\end{proof}

\subsection{Path-weighted distance and edge representation}
\label{sec:path-distance}

Let $t = (V_t, E_t, S_t)$ be a rooted binary tree as in
\cref{lem:path-decomposition}. Each node $v \in V_t$ is assigned a
nonnegative weight
\[
  \nodew_t(v) \ge 0.
\]
In applications, $\nodew_t(v)$ may be an impurity decrease, an
entropy gain, a variance decrease, or another nonnegative relevance
score attached to the node; see \cref{app:variance-gain-choice} for
the variance-gain choice. The canonical KPP embedding used in
this paper adopts the following convention.

\begin{assumption}[Zero leaf weights]
\label{ass:zero-leaf-weights}
For every leaf $v \in V_t^{\mathrm{leaf}}$,
\[
  \nodew_t(v) = 0.
\]
\end{assumption}

\begin{remark}[Why this convention is explicit]
\label{rem:zero-leaf-convention}
The original KPP construction uses a leaf coordinate of the form
$\sqrt{\nodew_t(\parnode(v)) / 2}\,\one\set{x \in S_v}$. This
coordinate does not include the leaf's own weight. Therefore, if
nonzero leaf weights are included in the path distance, the
canonical embedding is not isometric without a modification.
\Cref{app:nonzero-leaf-weights} gives a simple modified embedding
that handles nonzero leaf weights.
In the main text we keep \cref{ass:zero-leaf-weights}, which is
natural for split-gain weights because leaves do not split.
\end{remark}

\begin{definition}[Tree path distance]
Under \cref{ass:zero-leaf-weights}, the KPP path distance induced by
tree $t$ is
\begin{equation}
  \dt(x, x') \;:=\; \sum_{v \in \sigma_t(x, x')} \nodew_t(v),
  \qquad x, x' \in \Xspace.
  \label{eq:tree-path-distance}
\end{equation}
For a finite forest $T$, define the additive forest distance
\begin{equation}
  \dT(x, x') \;:=\; \sum_{t \in T} \dt(x, x').
  \label{eq:forest-path-distance}
\end{equation}
\end{definition}

The word ``distance'' should be read as pseudo-distance: two distinct
points can have zero distance if the forest does not separate them,
or if all separating nodes have zero weight.

\paragraph{From node weights to edge weights.}
For every non-root node $v \ne r_t$, define the edge weight
associated with the edge $(\parnode(v), v)$ by
\begin{equation}
  \edgew_t(v)
  \;:=\;
  \frac{\nodew_t(\parnode(v)) + \nodew_t(v)}{2}.
  \label{eq:edge-weight}
\end{equation}
Let $P_t(x, x')$ be the set of non-root nodes $v$ such that the edge
$(\parnode(v), v)$ lies on the unique simple path between the leaves
$\ell_t(x)$ and $\ell_t(x')$. Equivalently, $v \in P_t(x, x')$ if
and only if exactly one of the two leaves lies in the subtree rooted
at $v$.

\begin{lemma}[Node path as an edge sum]
\label{lem:node-path-edge-sum}
Under \cref{ass:zero-leaf-weights}, for all $x, x' \in \Xspace$,
\begin{equation}
  \dt(x, x') \;=\; \sum_{v \in P_t(x, x')} \edgew_t(v).
\end{equation}
\end{lemma}

\begin{proof}
Let $u_0 = \ell_t(x), u_1, \dots, u_L = \ell_t(x')$ be the ordered
nodes on the simple path between the two leaves. The endpoints
$u_0$ and $u_L$ are leaves, so $\nodew_t(u_0) = \nodew_t(u_L) = 0$ by
\cref{ass:zero-leaf-weights}. Summing the edge weights along the
path gives
\begin{align}
  \sum_{j=1}^{L} \frac{\nodew_t(u_{j-1}) + \nodew_t(u_j)}{2}
  &= \frac{\nodew_t(u_0) + \nodew_t(u_L)}{2}
     + \sum_{j=1}^{L-1} \nodew_t(u_j) \\
  &= \sum_{j=1}^{L-1} \nodew_t(u_j).
\end{align}
Since the endpoints have zero weight, the last expression is exactly
$\sum_{v \in \sigma_t(x, x')} \nodew_t(v) = \dt(x, x')$.
\end{proof}

\subsection{Node-indexed embedding and squared-Euclidean isometry}
\label{sec:embedding-isometry}

\begin{definition}[Tree embedding]
\label{def:tree-embedding}
Define $\phit : \Xspace \to \R^{\abs{V_t}}$ coordinatewise by
\begin{align}
  \phi_{t, r_t}(x) &:= \sqrt{\frac{\nodew_t(r_t)}{2}}, \\
  \phi_{t, v}(x)   &:= \sqrt{\edgew_t(v)}\,\one\set{x \in S_v},
    \qquad v \ne r_t.
\end{align}
For a finite forest $T$, define the forest embedding by concatenation:
\begin{equation}
  \phiT(x) \;:=\; \bigl(\phit(x)\bigr)_{t \in T}.
  \label{eq:forest-embedding}
\end{equation}
\end{definition}

The root coordinate is constant. It is kept because it makes the
embedding directly comparable to the original KPP coordinate
formula, but it has no effect on distances.

Coordinates are kept for all non-root nodes $v \ne r_t$, including
leaves. Under \cref{ass:zero-leaf-weights} (zero leaf weights), the
edge weight $\edgew_t(v) = (\nodew_t(\parnode(v)) + \nodew_t(v))/2$
reduces to $\nodew_t(\parnode(v))/2$ for leaves, so the leaf
coordinate $\phi_{t,v}(x) = \sqrt{\edgew_t(v)}\,\one\set{x \in S_v}$
carries the parent edge weight; this is necessary for the
squared-Euclidean path-distance isometry of
\cref{thm:kpp-isometry-tree}.

\begin{theorem}[KPP isometry, tree level]
\label{thm:kpp-isometry-tree}
Under \cref{ass:zero-leaf-weights}, for every $x, x' \in \Xspace$,
\begin{equation}
  \norm{\phit(x) - \phit(x')}_2^2 \;=\; \dt(x, x').
\end{equation}
\end{theorem}

\begin{proof}
The root coordinate is constant and cancels in the difference. For
a non-root node $v$,
\begin{equation}
  \bigl(\phi_{t, v}(x) - \phi_{t, v}(x')\bigr)^2
  \;=\;
  \edgew_t(v)\,\one\set{\one\set{x \in S_v} \ne \one\set{x' \in S_v}}.
\end{equation}
The indicator is one precisely when the edge $(\parnode(v), v)$
separates the two leaves, i.e.\ when $v \in P_t(x, x')$. Summing
over all non-root nodes and applying \cref{lem:node-path-edge-sum}
proves the claim.
\end{proof}

\begin{corollary}[KPP isometry, forest level]
\label{cor:kpp-isometry-forest}
For every finite forest $T$ satisfying \cref{ass:zero-leaf-weights}
tree by tree,
\begin{equation}
  \norm{\phiT(x) - \phiT(x')}_2^2 \;=\; \dT(x, x').
\end{equation}
\end{corollary}

\begin{proof}
The forest embedding is a concatenation, hence an orthogonal direct
sum over trees:
\[
  \norm{\phiT(x) - \phiT(x')}_2^2
  \;=\; \sum_{t \in T} \norm{\phit(x) - \phit(x')}_2^2
  \;=\; \sum_{t \in T} \dt(x, x')
  \;=\; \dT(x, x').
\]
\end{proof}

\subsection{Forest normalization, Gram, and distance reconstruction}
\label{sec:gram-distance}

For a single tree, define its total node mass
\begin{equation}
  A_t \;:=\; \sum_{v \in V_t} \nodew_t(v).
  \label{eq:tree-mass}
\end{equation}
For a finite forest, define
\begin{equation}
  \ST \;:=\; \sum_{t \in T} A_t \;=\; \sum_{t \in T} \sum_{v \in V_t} \nodew_t(v).
  \label{eq:forest-mass}
\end{equation}
Assume $\ST > 0$. The globally normalized forest embedding is
\begin{equation}
  \Phinode(x) \;:=\; \frac{\phiT(x)}{\sqrt{\ST}}.
  \label{eq:forest-embedding-normalized}
\end{equation}
Then, by \cref{cor:kpp-isometry-forest},
\begin{equation}
  \norm{\Phinode(x) - \Phinode(x')}_2^2
  \;=\; \frac{\dT(x, x')}{\ST}
  \;=:\; \deltaT(x, x').
  \label{eq:normalized-isometry}
\end{equation}

\begin{proposition}[Norm bound]
\label{prop:norm-bound}
Under \cref{ass:zero-leaf-weights}, for every finite forest $T$ with
$\ST > 0$,
\begin{equation}
  \sup_{x \in \Xspace} \norm{\Phinode(x)}_2 \;\le\; 1.
\end{equation}
\end{proposition}

\begin{proof}
The proof is given in \cref{app:norm-bound}. It is a direct
telescoping calculation along the root-to-leaf path followed by $x$
in each tree.
\end{proof}

\begin{remark}[Two normalizations used later]
\label{rem:two-normalizations}
For a fixed trained forest, the global normalization above is the
natural one: it gives a single feature vector $\Phinode(x)$ for
learning. For concentration over randomized trees, it is often
cleaner to normalize each tree separately by $A_t$ and then average
its Gram contribution. \Cref{sec:gram-concentration} makes this
distinction explicit to avoid mixing a random global denominator
with treewise averages.
\end{remark}

\paragraph{Fixed-forest Gram and distance reconstruction.}
Given training inputs $x_1, \dots, x_n$, the globally normalized KPP
Gram matrix of a fixed forest $T$ is
\begin{equation}
  (\Knode)_{ij} \;:=\; \inner{\Phinode(x_i)}{\Phinode(x_j)},
  \qquad 1 \le i, j \le n.
  \label{eq:knode-def}
\end{equation}
The normalized KPP
distance on the training sample is reconstructed from $\Knode$ by
\begin{equation}
  \deltaT(x_i, x_j)
  \;=\; (\Knode)_{ii} + (\Knode)_{jj} - 2(\Knode)_{ij}.
  \label{eq:distance-from-knode}
\end{equation}
This identity is deterministic and follows by expanding the squared
norm $\norm{\Phinode(x_i) - \Phinode(x_j)}_2^2$.

\begin{corollary}[Gram positive semidefiniteness and conditional non-diagonality]
\label{cor:gram-psd-nondiagonal}
The fixed-forest Gram $\Knode$ of \cref{eq:knode-def} is positive
semidefinite: writing $\Phi \in \R^{n \times m}$ for the matrix with
rows $\Phinode(x_i)^\top$, one has $\Knode = \Phi \Phi^\top$ by
\cref{def:tree-embedding,cor:kpp-isometry-forest}, so $\Knode$ is a
Gram matrix. Moreover, expanding \cref{eq:knode-def} over the node
coordinates gives, for all $i, j$,
\begin{equation}
  (\Knode)_{ij}
  \;=\; \frac{1}{\ST} \sum_{t \in T}
    \biggl(
      \frac{\nodew_t(r_t)}{2}
      + \sum_{v \in V_t \setminus \set{r_t}}
        \edgew_t(v)\,
        \one\set{x_i \in S_v}\,\one\set{x_j \in S_v}
    \biggr),
  \label{eq:knode-offdiag}
\end{equation}
a sum of nonnegative terms in which the indicator product is one
exactly when $v$ lies on both root-to-leaf paths, i.e.\ when $v$ is a
node common to $x_i$ and $x_j$. Hence $(\Knode)_{ij} \ge 0$, with
$(\Knode)_{ij} > 0$ if and only if $x_i$ and $x_j$ share at least one
common node of positive weight. Consequently $\Knode$ is non-diagonal
as soon as two distinct training points share a common node of
positive weight --- in particular the root $r_t$ whenever
$\nodew_t(r_t) > 0$; the off-diagonal entry $(\Knode)_{ij}$ vanishes
only when every node shared by $x_i$ and $x_j$ carries zero weight.
\end{corollary}

\begin{proof}
The factorisation $\Knode = \Phi \Phi^\top$ is immediate from
\cref{eq:knode-def}, and a Gram matrix is positive semidefinite.
\Cref{eq:knode-offdiag} follows by expanding the inner product
$\inner{\Phinode(x_i)}{\Phinode(x_j)}$ over the coordinates of
\cref{def:tree-embedding}, using
$\one\set{x_i \in S_v}\,\one\set{x_j \in S_v} = \one\set{x_i, x_j \in S_v}$;
the remaining claims read off the signs of the summands.
\end{proof}

Unlike a leaf-co-membership similarity, which records only
terminal-node agreement, $\Knode$ is an explicit node-level Gram.

\begin{remark}[Path-overlap granularity of the reconstructed distance]
\label{rem:path-overlap-distance}
The distance reconstructed in \cref{eq:distance-from-knode} resolves
how two inputs' root-to-leaf paths diverge. The constant root
coordinate cancels there, and by
\cref{thm:kpp-isometry-tree,cor:kpp-isometry-forest} the quantity
$\deltaT(x_i, x_j)$ is the normalized, edge-weighted length of the
internal-node paths separating $x_i$ and $x_j$ across the forest --- a
real-valued quantity determined by which internal nodes the two
trajectories share and by their edge weights. A leaf-co-membership
similarity is instead binary per tree, recording only whether two
inputs reach the same leaf; it therefore cannot distinguish inputs that
diverge high in a tree from inputs that diverge just above the leaves,
discarding the within-path structure that the edge-weighted
internal-node coordinates of \cref{def:tree-embedding} --- retained
alongside the leaf coordinates --- make explicit. Complementing the
Gram-side statement of \cref{cor:gram-psd-nondiagonal}, on the distance
side the forest aggregates these graded per-tree separations
(\cref{eq:forest-embedding-normalized}), whereas a leaf-co-membership
similarity would aggregate binary same-leaf indicators.
\end{remark}

\subsection{Finite-forest concentration}
\label{sec:gram-concentration}

For concentration over randomized trees, consider independent trees
$t_1, \dots, t_K$ generated conditionally on the training inputs.
For a single tree $t$, assume its total mass
$A_t = \sum_{v \in V_t} \nodew_t(v)$ is positive and define the
tree-normalized embedding
\begin{equation}
  \bar\phi_t(x) \;:=\; \frac{\phit(x)}{\sqrt{A_t}}.
  \label{eq:tree-normalized-embedding}
\end{equation}
Let $H_t \in \R^{n \times n}$ be the corresponding single-tree Gram,
\begin{equation}
  (H_t)_{ij} \;:=\; \inner{\bar\phi_t(x_i)}{\bar\phi_t(x_j)}.
  \label{eq:single-tree-gram}
\end{equation}
By \cref{prop:norm-bound} applied to a single tree,
$\norm{\bar\phi_t(x)}_2 \le 1$, hence $\abs{(H_t)_{ij}} \le 1$.

Define the averaged Gram and its infinite-randomization limit (conditional on the training inputs) by
\begin{equation}
  \GK \;:=\; \frac{1}{K} \sum_{k=1}^{K} H_{t_k},
  \qquad
  \Ginf \;:=\; \E\!\left[H_t \mid x_1, \dots, x_n\right].
  \label{eq:gk-ginf-def}
\end{equation}
The associated averaged normalized distance is
\begin{equation}
  \bar\delta_K(x, x')
  \;:=\; \frac{1}{K} \sum_{k=1}^{K} \frac{d_{t_k}(x, x')}{A_{t_k}}.
  \label{eq:averaged-distance}
\end{equation}
On the training sample, $\bar\delta_K$ is reconstructed from $\GK$
by the same formula:
\begin{equation}
  \bar\delta_K(x_i, x_j)
  \;=\; (\GK)_{ii} + (\GK)_{jj} - 2(\GK)_{ij}.
  \label{eq:averaged-distance-from-gk}
\end{equation}

\begin{remark}[Relation with the globally normalized forest Gram]
\label{rem:global-vs-averaged-gram}
The globally normalized Gram of the concatenated forest satisfies
\begin{equation}
  \Knode \;=\; \sum_{t \in T} \frac{A_t}{\ST}\, H_t.
  \label{eq:knode-as-weighted-average}
\end{equation}
Thus $\Knode$ is a mass-weighted average of the tree-normalized
Grams. If the tree masses $A_t$ are equal, then $\Knode = \GK$. In
general they are different but closely related objects. The
fixed-forest learning sections may use $\Knode$, whereas
concentration statements are cleaner for $\GK$.
\end{remark}

\paragraph{Entrywise concentration.}

\begin{proposition}[Entrywise finite-forest concentration]
\label{prop:entrywise-concentration}
Condition on the training inputs $x_1, \dots, x_n$. Assume
$t_1, \dots, t_K$ are independent under this conditioning, and
assume $\abs{(H_{t_k})_{ij}} \le 1$ almost surely for every
$i, j, k$. This independence holds when the trees are drawn independently under the
conditioning. Under label-independent randomization the tree draw does
not use the labels, so the trees are independent conditionally on the
inputs $x_1, \dots, x_n$ alone and $\Ginf$ coincides with
$\E[H_t \mid x_1, \dots, x_n]$. Under the honest or cross-fit boundaries
PB-02, PB-03 of \cref{sec:proof-boundary} the trees are grown on the
partition fold --- whose labels they do use --- but are independent of
the fit-fold labels conditionally on the partition-fold data, so $\Ginf$
is read as the corresponding partition-conditional infinite-randomization
Gram. For a label-dependent forest (PB-04) the same concentration holds
conditionally on the full training sample $(x_i, y_i)_{i=1}^n$, with
$\Ginf$ read as the corresponding infinite-randomization Gram; the
operator-norm diagnostic of \cref{rem:operator-norm-diagnostic} is read
under the same convention.
Then, for every fixed pair $(i, j)$ and every
$\epsilon > 0$,
\begin{equation}
  \Prob\!\left(\abs{(\GK)_{ij} - (\Ginf)_{ij}} \ge \epsilon\right)
  \;\le\; 2 \exp\!\left(-\frac{K \epsilon^2}{2}\right).
\end{equation}
Consequently, with probability at least $1 - \delta$,
\begin{equation}
  \max_{1 \le i, j \le n} \abs{(\GK)_{ij} - (\Ginf)_{ij}}
  \;\le\;
  \sqrt{\frac{2}{K} \log \frac{2 n^2}{\delta}}.
\end{equation}
\end{proposition}

\begin{proof}
For fixed $(i, j)$, the random variables $(H_{t_k})_{ij}$ are
conditionally independent and lie in $[-1, 1]$. Hoeffding's
inequality for variables with range length $2$ gives the first
display. The second follows by a union bound over the $n^2$ matrix
entries.
\end{proof}

\begin{corollary}[Uniform distance concentration from the Gram]
On the event of \cref{prop:entrywise-concentration}, the averaged
normalized distances satisfy
\begin{equation}
  \max_{1 \le i, j \le n}
  \abs{\bar\delta_K(x_i, x_j) - \bar\delta_\infty(x_i, x_j)}
  \;\le\;
  4 \sqrt{\frac{2}{K} \log \frac{2 n^2}{\delta}},
\end{equation}
where $\bar\delta_\infty$ is reconstructed from $\Ginf$.
\end{corollary}

\begin{proof}
For any Gram matrix $G$, the induced squared distance is
$G_{ii} + G_{jj} - 2 G_{ij}$. Therefore a uniform entrywise error at
most $\eta$ implies a distance error at most $4\eta$. Apply
\cref{prop:entrywise-concentration} with
$\eta = \sqrt{(2/K) \log(2 n^2 / \delta)}$.
\end{proof}

\begin{remark}[Operator-norm concentration and half-forest diagnostics]
\label{rem:operator-norm-diagnostic}
Matrix Bernstein inequalities \citep{Tropp2012MatrixBernstein} can provide operator-norm bounds for
$\norm{\GK - \Ginf}_{\op}$ under additional assumptions on the
per-tree Gram matrices. The half-forest discrepancy
\[
  \Deltak \;:=\; \frac{\norm{\GK^{(1)} - \GK^{(2)}}_{\op}}{\sqrt{2}}
\]
is useful as a measurable stability diagnostic. It should not be
substituted into a theorem as an upper bound on
$\norm{\GK - \Ginf}_{\op}$ unless a calibrated high-probability
comparison, with constants and probability level, has been proved
in the relevant section.
\end{remark}

\section{Exact additive attribution}
\label{sec:attribution}

Let $g(x) = b + \inner{w}{\Phinode(x)}$ be any linear model in the
normalized KPP representation, with $w \in \R^m$ and $b \in \R$.
Because the coordinates of $\Phinode$ are indexed by tree nodes, the
score $g$ admits an exact additive decomposition over nodes, and
hence over input variables.

\subsection{Node-level additive attribution}
\label{sec:attribution-node}

\begin{proposition}[Node-level additive attribution]
\label{prop:attribution-node}
For every $x \in \Xspace$,
\begin{equation}
  g(x) \;=\; b + \sum_{u} w_u \, \Phinode{}_{u}(x),
  \label{eq:attribution-node}
\end{equation}
where the sum runs over all coordinates $u$ of the forest embedding,
that is, over all tree nodes whose coordinates are retained by
\cref{def:tree-embedding}.
\end{proposition}

\begin{proof}
The displayed identity is the coordinate expansion of the inner
product $\inner{w}{\Phinode(x)}$ along the canonical basis of
$\R^m$.
\end{proof}

\subsection{Variable-level additive decomposition by marginal split}
\label{sec:attribution-variable}

\begin{proposition}[Variable-level additive decomposition by marginal split]
\label{prop:attribution-variable}
Let $\operatorname{split}(u)$ denote the input variable used at the internal
node $u$. For a non-root node $v$, the membership indicator $\one\set{x \in S_v}$
is governed by the split at the parent edge $(\parnode(v), v)$; write
$\operatorname{split}(\parnode(v))$ for that variable. Define the marginal-split
contribution
\begin{equation}
  \Psi_j^{\mathrm{edge}}(x) \;:=\; \sum_{\substack{v \text{ not a root}\\ \operatorname{split}(\parnode(v)) = j}}
                    w_v \, \Phinode{}_{v}(x), \qquad j = 1, \dots, p,
  \label{eq:variable-contribution}
\end{equation}
the sum running over all non-root nodes of the forest, leaves included. Then,
for every $x \in \Xspace$,
\begin{equation}
  g(x) \;=\; c_0 + \sum_{j=1}^{p} \Psi_j^{\mathrm{edge}}(x), \qquad
  c_0 \;:=\; b + \sum_{v \text{ a root}} w_v \, \Phinode{}_{v}(x),
  \label{eq:attribution-variable}
\end{equation}
where $c_0$ collects the intercept and the constant root coordinates, the only
coordinates with no parent-edge split.
\end{proposition}

\begin{proof}
By \cref{prop:attribution-node}, $g(x) = b + \sum_{v} w_v \Phinode{}_{v}(x)$ over
all retained coordinates. Each non-root coordinate $\Phinode{}_{v}$ is governed
by the split at its parent edge $(\parnode(v), v)$, hence is associated with
exactly one variable $\operatorname{split}(\parnode(v))$; the sum over non-root
nodes therefore partitions into $\sum_{j=1}^{p} \Psi_j^{\mathrm{edge}}(x)$. The
root coordinates are constant and have no parent edge; they are collected with
the intercept into $c_0$.
\end{proof}

\paragraph{Centered variant.}
A centered variant of \cref{eq:attribution-node} uses the empirical
baseline $\bar\phi := n^{-1} \sum_{i=1}^{n} \Phinode(x_i)$:
\begin{equation}
  g(x) \;=\; b + \inner{w}{\bar\phi}
              + \sum_{u} w_u \bigl(\Phinode{}_{u}(x) - \bar\phi_u\bigr).
  \label{eq:attribution-centered}
\end{equation}
This is the exact linear-attribution form in the KPP feature space:
the intercept-like term $b + \inner{w}{\bar\phi}$ is constant in $x$,
and each node contributes its own
$w_u\bigl(\Phinode{}_u(x) - \bar\phi_u\bigr)$ to the deviation of
$g(x)$ from the dataset mean.

\begin{remark}[What the marginal-split decomposition requires]
\label{rem:attribution-split-coordinates}
The decomposition of \cref{prop:attribution-variable} is a functional of the
forest coordinates together with the parent-edge map $v \mapsto \operatorname{split}(\parnode(v))$
supplied by the tree structure. Every non-root coordinate, leaves included, is
assigned to the variable of its parent edge; only the constant root coordinates
carry no split and remain in $c_0$. The decomposition therefore needs more than
the leaf-membership pattern: it needs the tree structure that records which split
governs each node. A leaf-only encoding, which retains only the indicators
$\one\set{x \in S_\ell}$ of the terminal cells without the edge structure that
labels their ancestor splits, admits no analogous variable-level decomposition,
even though it can reproduce the predictor itself.
\end{remark}

\begin{remark}[Identifiability: the decomposition is parametrization-dependent]
\label{rem:attribution-identifiability}
The node coordinates are linearly dependent: at any internal node the two child
indicators sum to the parent indicator, $\one\set{x \in S_a} + \one\set{x \in S_b}
= \one\set{x \in S_{\parnode(a)}}$, and the root indicator is constant. The KPP
design is therefore overcomplete and the weight vector $w$ is not unique --- many
$w$ realise the same predictor $g$. Both the node-level identity of
\cref{prop:attribution-node} and the marginal-split decomposition of
\cref{prop:attribution-variable} are exact for any such $w$, so the predictor is
decomposed exactly regardless; but the particular contributions $w_v \Phinode{}_v$
and $\Psi_j^{\mathrm{edge}}$ are those of the minimum-norm solution selected by the
regularized objective, not an intrinsic property of the function. Reported
attributions should therefore be read as attributions of the regularized
parametrization.
\end{remark}

\begin{remark}[Marginal-split, not axiomatic, attribution]
\label{rem:attribution-not-axiomatic}
The decomposition is marginal in the terminal split: $\Psi_j^{\mathrm{edge}}$
collects the coordinates whose parent edge tests variable $j$, so it attributes
to the last split on the path to each node rather than allocating value over
coalitions of variables. Two consequences distinguish it from an axiomatic
attribution such as a Shapley value. First, membership $\one\set{x \in S_v}$
encodes the conjunction of all splits from the root to $v$, so the contribution
of a variable is conditional on the variables tested above it and carries no
symmetry or dummy guarantee. Second, the edge weight $\edgew_t(v)$ averages the
parent and child node weights, so the scale of a contribution is not intrinsic
to the variable. The decomposition is therefore an exact, path-dependent
accounting of the fitted predictor, not an axiomatic feature attribution, and
should not be read as a Shapley or SHAP decomposition.
\end{remark}

\section{Deterministic Lipschitz robust radius}
\label{sec:robust-radius}

The KPP metric provides a deterministic Lipschitz certificate for any
linear score in the normalized representation. The certificate is in
the embedding metric $\sqrt{\deltaT}$ (whose square is the path distance
$\deltaT$), not in the raw input metric on $\Xspace$;
it should not be confused with robustness to arbitrary $\ell_p$
perturbations of the input.

\begin{proposition}[Lipschitz control in the KPP metric]
\label{prop:lipschitz-kpp}
Let $g(x) = b + \inner{w}{\Phinode(x)}$ be a linear model in the
normalized KPP representation. For all $x, x' \in \Xspace$,
\begin{equation}
  \abs{g(x) - g(x')}
  \;\le\;
  \norm{w}_2 \sqrt{\deltaT(x, x')}.
  \label{eq:lipschitz-kpp}
\end{equation}
\end{proposition}

\begin{proof}
By Cauchy--Schwarz and the normalized isometry of
\cref{eq:normalized-isometry},
\begin{equation*}
  \abs{g(x) - g(x')}
  \;=\; \abs{\inner{w}{\Phinode(x) - \Phinode(x')}}
  \;\le\; \norm{w}_2 \norm{\Phinode(x) - \Phinode(x')}_2
  \;=\; \norm{w}_2 \sqrt{\deltaT(x, x')}.
\end{equation*}
\end{proof}

\begin{remark}[The certificate is in the embedding metric, not its square]
\label{rem:deltaT-half-metric}
The induced metric is the embedding distance
$\sqrt{\deltaT(x, x')} = \norm{\Phinode(x) - \Phinode(x')}_2$, of which
$\deltaT$ is the square. Although $\deltaT$ is thus a squared Euclidean
distance, it is itself a pseudometric: by \cref{thm:kpp-isometry-tree}
it equals a non-negative weighted sum of cut pseudometrics, hence
satisfies the triangle inequality, and both $\deltaT$ and its square
root $\sqrt{\deltaT}$ are pseudometrics. The certificate
\cref{eq:lipschitz-kpp} is thus $\norm{w}_2$-Lipschitz in
$\sqrt{\deltaT}$, equivalently H\"older-$1/2$ in $\deltaT$. The robust
radius is correspondingly $m / \norm{w}_2$ in $\sqrt{\deltaT}$, i.e.\
$(m / \norm{w}_2)^2$ in $\deltaT$, consistent with the thresholds
$\deltaT(x, x') < (m / \norm{w}_2)^2$ used below.
\end{remark}

In binary classification with $y \in \set{-1, +1}$, if the geometric
margin $m(x) = y\, g(x)$ is positive, then any $x'$ satisfying
$\deltaT(x, x') < (m(x) / \norm{w}_2)^2$ preserves the predicted
sign --- this is formalised in
\cref{sec:robust-radius-classification}. In regression, the same
inequality gives output stability:
$\deltaT(x, x') < (\epsilon / \norm{w}_2)^2$ implies
$\abs{g(x) - g(x')} < \epsilon$ --- this underlies the
$\mathrm{RobMSE}$ certificate in
\cref{sec:robust-radius-regression}.

\subsection{Regression: robust MSE certificate}
\label{sec:robust-radius-regression}

Let $L := \norm{\widehat w_{\lreg}}_2$. By \cref{prop:lipschitz-kpp}
applied to the ridge predictor $\widehat f(x) = \inner{\widehat w_{\lreg}}{\Phinode(x)}$,
\begin{equation}
  \abs{\widehat f(x) - \widehat f(x')}
  \;\le\;
  L \sqrt{\deltaT(x, x')}.
  \label{eq:reg-lipschitz}
\end{equation}
Thus if $\deltaT(x, x') < (\epsilon / L)^2$, then
$\abs{\widehat f(x) - \widehat f(x')} < \epsilon$.
We take $L > 0$ throughout; if $L = 0$ the predictor $\widehat f$ is
constant (as is the classification score $g_w$ when $w = 0$, giving a
constant margin), every prediction is invariant across any
$\deltaT$-ball, and every radius is trivially certified, with
$u_{\max}$ read as $+\infty$.

For residuals $e_i := y_i - \widehat f(x_i)$ on a sample $\set{(x_i, y_i)}_{i=1}^n$, a pointwise
robust squared-error upper bound at metric radius $u \in [0, 1]$ is
\begin{equation}
  \bigl(y_i - \widehat f(x')\bigr)^2
  \;\le\;
  e_i^2 + 2 \abs{e_i} L \sqrt{u} + L^2 u,
  \qquad
  \deltaT(x_i, x') \le u.
  \label{eq:reg-pointwise-robust}
\end{equation}
Define the empirical squared and absolute errors
\begin{equation}
  \mathrm{MSE} := \frac{1}{n} \sum_{i=1}^{n} e_i^2,
  \qquad
  \mathrm{MAE} := \frac{1}{n} \sum_{i=1}^{n} \abs{e_i}.
  \label{eq:mse-mae}
\end{equation}
Define the worst-case fixed-label error
$\mathrm{RobMSE}(u) := \frac{1}{n} \sum_{i=1}^{n}
\sup_{\deltaT(x_i, x') \le u} \bigl(y_i - \widehat f(x')\bigr)^2$.
Averaging \cref{eq:reg-pointwise-robust} gives the computable upper bound
\begin{equation}
  \mathrm{RobMSE}(u)
  \;\le\;
  \mathrm{MSE} + 2\,\mathrm{MAE}\,L \sqrt{u} + L^2 u.
  \label{eq:robust-mse}
\end{equation}
This is a deterministic certificate in the KPP metric; it should not
be confused with robustness to arbitrary perturbations in the raw
input metric. This curve measures the stability of the prediction against the fixed observed label: it equals robustness of the risk only when the true response is unchanged across the $\deltaT$-ball, and otherwise quantifies output stability rather than error degradation.

\subsection{Classification: robust radius and margin}
\label{sec:robust-radius-classification}

\paragraph{Deterministic margin robustness.}
Let
\[
  g_w(x) = \inner{w}{\Phinode(x)}
  \quad \text{and} \quad
  m_w(x, y) := y\, g_w(x)
\]
denote the score and the (geometric) margin in the KPP feature space,
with binary labels $y \in \set{-1, +1}$. The Lipschitz inequality
\cref{eq:lipschitz-kpp} specialised to $g_w$ reads
\begin{equation}
  \abs{g_w(x) - g_w(x')}
  \;\le\;
  \norm{w}_2 \sqrt{\deltaT(x, x')}.
  \label{eq:classification-kpp-lipschitz}
\end{equation}

\begin{proposition}[Certified KPP robust radius]
\label{prop:classification-robust-radius}
If $m_w(x, y) > 0$, then every $x' \in \Xspace$ satisfying
\begin{equation}
  \deltaT(x, x')
  \;<\;
  \left(\frac{m_w(x, y)}{\norm{w}_2}\right)^{2}
  \label{eq:classification-robust-radius}
\end{equation}
preserves the predicted sign: $\sign(g_w(x')) = \sign(g_w(x))$.
\end{proposition}

\begin{proof}
By \cref{eq:classification-kpp-lipschitz},
$\abs{g_w(x') - g_w(x)} \le \norm{w}_2 \sqrt{\deltaT(x, x')}$. The
hypothesis $\deltaT(x, x') < (m_w(x, y) / \norm{w}_2)^2$ implies
$\abs{g_w(x') - g_w(x)} < m_w(x, y) = y\, g_w(x)$. Therefore
\begin{equation*}
  y\, g_w(x')
  \;=\; y\, g_w(x) + y\bigl(g_w(x') - g_w(x)\bigr)
  \;\ge\; y\, g_w(x) - \abs{g_w(x') - g_w(x)}
  \;>\; 0,
\end{equation*}
which gives $\sign(g_w(x')) = y = \sign(g_w(x))$.
\end{proof}

The reference implementation evaluates the affine score $g_w(x) + b$
with a constant intercept $b$ (\cref{sec:setup-risk}). Because $b$ is
constant it cancels in $g_w(x) - g_w(x')$, so the Lipschitz constant
$\norm{w}_2$ and hence the radius formula are unchanged; the same
certificate holds for the deployed affine classifier with the margin
taken as $y\,(g_w(x) + b)$. The reported robust-accuracy curves use
this affine margin.

The empirical robust-accuracy curve at KPP radius $u \in [0, 1]$ is
\begin{equation}
  \mathrm{RobAcc}(u)
  \;:=\;
  \frac{1}{n} \sum_{i=1}^{n}
    \one\set{m_i > \norm{\widehat w}_2 \sqrt{u}},
  \label{eq:robust-accuracy-curve}
\end{equation}
where $m_i := m_{\widehat w}(x_i, y_i) = y_i\, g_{\widehat w}(x_i)$.
This is a deterministic certificate in the KPP metric. It is not a
statement about arbitrary perturbations in the raw input metric.
This curve certifies invariance of the predicted sign within the
$\deltaT$-ball: it equals robustness of the classification accuracy
only when the true label is unchanged across the ball, and otherwise
quantifies prediction stability rather than robust correctness.

\paragraph{What the KPP metric does not imply.}
The Lipschitz inequality \cref{eq:classification-kpp-lipschitz} gives
an upper bound on score variation. It implies that points close to a
high-margin point preserve the prediction. It does not imply that a
point with small score must be close to the Bayes boundary in
$\deltaT$. Therefore a geometric mass condition near the Bayes
boundary does not by itself imply a tail bound for
$\abs{g_w(X)}$. The invalid implication is
\[
  \Prob\bigl(\deltaT(X, \mathcal B) \le r\bigr) \le C r^\nu
  \quad \nRightarrow \quad
  \Prob\bigl(\abs{g_w(X)} \le \gamma\bigr) \le C' \gamma^\alpha.
\]
A score can be nearly zero far from the Bayes boundary unless a
lower-margin condition is added separately.

\paragraph{Safe score-margin assumption.}
A safe condition for fast-rate discussions is the direct score-margin
assumption.

\begin{assumption}[Direct score-margin tail]
\label{ass:direct-score-margin}
For a comparator score
$g_\star(x) = \inner{w_\star}{\Phinode(x)}$, there exist constants
$C > 0$ and $\alpha > 0$ such that, for all $\gamma > 0$,
\begin{equation}
  \Prob\bigl(\abs{g_\star(X)} \le \gamma\bigr)
  \;\le\;
  C \gamma^\alpha.
  \label{eq:direct-score-margin}
\end{equation}
\end{assumption}

Under \cref{ass:direct-score-margin}, one may import standard
localized-complexity or margin-based fast-rate results for logistic
or hinge classification, provided the representation is fixed or
honest and all approximation and regularisation terms are specified.
This paper does not present such a result as a proved KPP theorem;
\cref{ass:direct-score-margin} is recorded so that any future
fast-rate theorem in this representation has a valid starting point.
Such a result would additionally require a Bernstein/Tsybakov-type
noise condition linking the comparator to $\eta$ (for instance
$g_\star = \log(\eta/(1-\eta))$); the score-margin tail of
\cref{ass:direct-score-margin} alone does not control the conditional
risk curvature near the Bayes boundary.

\section{Rademacher bounds and partition oracles}
\label{sec:rademacher}

\subsection{Asymmetry of regression and classification under unconstrained ERM (radical honesty)}
\label{sec:radical-honesty}
The trace Rademacher bound of \cref{prop:trace-rademacher} and the
Lipschitz contraction step that turns it into a uniform loss bound
apply identically to the squared loss and to the logistic surrogate:
in both cases the score complexity is controlled by
$B\sqrt{\Tr(\Knode)}/n$, and the loss class is Lipschitz in the
score, with constant $2(M + B)$ for the squared loss on
$[-(M + B), M + B]$ and constant $1$ for the logistic surrogate
together with the bounded-differences range $C_B = \log(1 + e^B)$.
The asymmetry between
\cref{sec:rademacher-regression,sec:rademacher-classification} is
therefore not in the proof but in what the resulting uniform bound
delivers. For regression, $\Risk(g_w)$ is the squared error itself,
so \cref{prop:uniform-squared-loss-bound} directly controls the prediction risk in terms of the empirical risk. For classification, $\Risk(g_w)$ is the logistic
risk, and converting it into a guarantee on the $0/1$ error requires
the pointwise calibration of \cref{eq:logistic-calibration}; sharper
fast-rate refinements require, in addition, a tail condition on
$\abs{g_w(X)}$ that does not follow from the KPP geometry. We
present
\cref{prop:uniform-squared-loss-bound,thm:uniform-logistic-bound} as
a single methodological frame rather than as two separate
KPP-specific theorems, and isolate in the present subsection two
obstructions to going further on the classification side under
unconstrained ERM. We refer to them as M-C-01 (margin) and O-C-02
(oracle decomposition).

\paragraph{Geometric margin does not imply score margin (M-C-01).}
For a binary classifier of the form $x \mapsto \sign(g_w(x))$,
fast-rate excess-risk statements rest on a tail condition on the
score itself, $\Prob(\abs{g_w(X)} \le \gamma) \le C' \gamma^\alpha$,
not on a geometric-margin condition
$\Prob(\deltaT(X, \mathcal B) \le r) \le C r^\nu$ in the KPP path
metric. As recorded in \cref{sec:robust-radius-classification} and
made concrete by the worked-out one-dimensional counterexample of
\cref{app:counterexample-margin}, a $\sqrt{\deltaT}$-Lipschitz score can be
small far from its zero set, so the first implication does not
follow from the second. The safe replacement is
\cref{ass:direct-score-margin}, the direct score-margin tail
condition. This is the obstruction we anchor as M-C-01: the KPP
geometry alone is not sufficient to upgrade the uniform logistic
bound of \cref{thm:uniform-logistic-bound} into a fast-rate
classification statement.

\paragraph{Raw forest sums are not a forest-level oracle (O-C-02).}
At the tree level, \cref{sec:partition-oracles} expresses the
leaf-wise oracle risk as a root quantity minus a telescope of local
refinement gains, with the same identity for variance (regression)
and entropy (classification). At the forest level, the natural
invariants are the tree-wise averages $IG_{\mathrm{var}}(T)$ and
$IG_h(T)$. As recorded in \cref{rem:raw-forest-sum-pitfall}, a raw
sum of per-tree refinement gains over all internal nodes of all
trees is not the refinement gain of any single forest partition and
scales with the number of trees: duplicating identical trees
multiplies it by $K$, while the partition itself does not change.
We anchor this as O-C-02. Under the honesty separation of
\cref{sec:proof-boundary}, the empirical analogues
$\widehat{IG}_{\mathrm{var}}(T)$ and $\widehat{IG}_h(T)$ are
partition-quality diagnostics whose finite-sample concentration is
not proved in this paper.

\paragraph{What remains open: FR-R-01 and FR-C-01.}
Two fast-rate statements would be natural continuations of the
present work. FR-R-01 is a fast-rate regression statement under a
Bernstein-type or strong-convexity condition specific to the KPP
design, which would sharpen the $1/\sqrt n$ uniform risk bound of
\cref{sec:rademacher-regression}. FR-C-01 is the classification
analogue under \cref{ass:direct-score-margin}, which would combine
the direct score-margin tail with localized-complexity machinery to
obtain a fast classification rate. Both require additional
assumptions on the data distribution and additional technical
apparatus, none of which is specific to KPP; they are not proved in
this paper. They are listed in \cref{sec:discussion} as future
work.

\subsection{Trace-based Rademacher bound and effective dimension (regression)}
\label{sec:rademacher-regression}

The results of this subsection are conditional on the KPP
representation in the following sense. The feature map $\Phinode$ is
frozen once the forest has been fit, so the representation is held
fixed as a map and is decoupled from the labels used for the final
estimator (the fixed or honest regime, PB-01 and PB-02). The
probability is taken over the i.i.d.\ draw of the fit sample on which
the empirical risk and the empirical Rademacher complexity are
evaluated, and the data-dependent complexity $\Tr(\Knode)$ of that
realized sample is what enters the bound. In particular the realized
training features $\Phinode(x_i)$ are not conditioned on as fixed
constants, so the guarantee is a sampling statement over the fit
sample rather than a fixed-design statement conditional on
$x_1,\dots,x_n$, the latter being the setting of the fixed-design
bias--variance decomposition treated separately. The same statements apply to honest and cross-fit designs, with
different conditioning. Under an honest design (PB-02), conditionally on
the fold that built the representation the map is fixed and the
statements apply verbatim to the i.i.d.\ fit fold. Under a cross-fit
design (PB-03) the per-fold representations are fold-dependent, so the
statements hold fold by fold, conditionally on the out-of-fold data,
with a union bound over the $Q$ folds distributing $\delta$; a single
predictor refit on the stacked out-of-fold embeddings is not covered by
these statements as they stand.

For the trace-based Rademacher bound \citep{BartlettMendelson2002Rademacher} it is convenient to work with the
homogeneous linear sub-class of \cref{eq:hcal-B},
\begin{equation}
  \Fcal_B \;:=\; \set{x \mapsto \inner{w}{\Phinode(x)} : \norm{w}_2 \le B}
                 \;\subset\; \Hcal_B,
  \label{eq:fcal-B-linear}
\end{equation}
that is, $\Hcal_B$ with the bias term set to $b = 0$.

\begin{proposition}[Trace Rademacher bound]
\label{prop:trace-rademacher}
Conditionally on the KPP representation, the empirical Rademacher
complexity of $\Fcal_B$ on the training inputs satisfies
\begin{equation}
  \widehat{\Rcal}_n(\Fcal_B)
  \;\le\;
  \frac{B}{n}\sqrt{\Tr(\Knode)}
  \;\le\;
  \frac{B}{\sqrt n}.
  \label{eq:trace-rademacher}
\end{equation}
\end{proposition}

\begin{proof}
Let $z_i = \Phinode(x_i)$. Then
\begin{align}
  \widehat{\Rcal}_n(\Fcal_B)
  &= \frac{1}{n}\,\E_\sigma \sup_{\norm{w}_2 \le B}
     \sum_{i=1}^{n} \sigma_i \inner{w}{z_i} \\
  &= \frac{B}{n}\,\E_\sigma \norm{\sum_{i=1}^{n} \sigma_i z_i}_{2} \\
  &\le \frac{B}{n}\sqrt{\E_\sigma \norm{\sum_{i=1}^{n} \sigma_i z_i}_{2}^{2}} \\
  &= \frac{B}{n}\sqrt{\sum_{i=1}^{n} \norm{z_i}_2^2}
   \;=\; \frac{B}{n}\sqrt{\Tr(\Knode)}.
\end{align}
The final inequality $\Tr(\Knode) \le n$ uses
$\norm{\Phinode(x_i)}_2 \le 1$ from \cref{prop:norm-bound}.
\end{proof}

If the response is bounded, the trace bound translates into a uniform
squared-loss bound for regression.

\begin{proposition}[Uniform squared-loss bound]
\label{prop:uniform-squared-loss-bound}
Assume $\abs{Y} \le M$ almost surely and condition on the KPP
representation in the sense of \cref{sec:proof-boundary}. For any $B > 0$ and any $\delta \in (0, 1)$, with
probability at least $1 - \delta$, simultaneously for all $w$ with
$\norm{w}_2 \le B$,
\begin{equation}
  \Risk(g_w)
  \;\le\;
  \Riskhat(g_w)
  + 4 (M + B) \frac{B}{n}\sqrt{\Tr(\Knode)}
  + 3 (M + B)^2 \sqrt{\frac{2 \log(2/\delta)}{n}},
  \label{eq:uniform-squared-loss-bound}
\end{equation}
where $g_w(x) = \inner{w}{\Phinode(x)}$ and $\Risk, \Riskhat$ are the
population and empirical squared-loss risks of \cref{eq:risk-emp-risk}.
\end{proposition}

\begin{proof}
On $\norm{w}_2 \le B$, predictions satisfy
$\abs{\inner{w}{\Phinode(x)}} \le B$, so the squared loss
$(y - g)^2$ is $2(M + B)$-Lipschitz in $g$ on the range
$[-(M + B), M + B]$. Symmetrization, the Lipschitz contraction
inequality, \cref{prop:trace-rademacher}, and a bounded-differences
step with range $(M + B)^2$ give the claim.
\end{proof}

\paragraph{Effective dimension and ridge smoother.}
Let $\mueig_1, \dots, \mueig_n$ be the eigenvalues of $\Knode$. For
$\lreg > 0$, define the ridge effective dimension
\begin{equation}
  \deff(\lreg)
  \;:=\;
  \Tr\!\bigl(\Knode (\Knode + \lreg I_n)^{-1}\bigr)
  \;=\;
  \sum_{j=1}^{n} \frac{\mueig_j}{\mueig_j + \lreg},
  \label{eq:effective-dimension}
\end{equation}
the degrees of freedom of the ridge smoother
$S_{\lreg} := \Knode (\Knode + \lreg I_n)^{-1}$.

\begin{proposition}[Fixed-design ridge bias--variance decomposition]
\label{prop:ridge-bias-variance}
Condition on $x_1, \dots, x_n$ and on the KPP representation, and
suppose
\begin{equation}
  y = f_\star + \varepsilon,
  \qquad
  \E[\varepsilon \mid x_1^n, \Phinode] = 0,
  \qquad
  \E[\varepsilon \varepsilon^\top \mid x_1^n, \Phinode] = \sigma^2 I_n.
\end{equation}
Then the fitted vector $\widehat y = S_{\lreg} y$ (see \cref{app:ridge-primal-dual}) satisfies
\begin{equation}
  \E\!\left[\frac{1}{n} \norm{\widehat y - f_\star}_2^2 \given x_1^n, \Phinode\right]
  \;=\;
  \frac{1}{n} \norm{(I_n - S_{\lreg}) f_\star}_2^2
  + \frac{\sigma^2}{n}\,\Tr(S_{\lreg}^2),
  \label{eq:ridge-bias-variance}
\end{equation}
and
\begin{equation}
  \Tr(S_{\lreg}^2) \;\le\; \Tr(S_{\lreg}) \;=\; \deff(\lreg).
  \label{eq:variance-deff}
\end{equation}
\end{proposition}

\begin{proof}
Write
$\widehat y - f_\star = (S_{\lreg} - I_n) f_\star + S_{\lreg} \varepsilon$
and expand the conditional squared norm using
$\E[\varepsilon \mid x_1^n, \Phinode] = 0$ and
$\E[\varepsilon \varepsilon^\top \mid x_1^n, \Phinode] = \sigma^2 I_n$, which
gives \cref{eq:ridge-bias-variance}. The eigenvalues of $S_{\lreg}$
are $h_j = \mueig_j / (\mueig_j + \lreg) \in [0, 1]$, hence
$h_j^2 \le h_j$, which proves \cref{eq:variance-deff}.
\end{proof}

\begin{remark}[What the effective dimension does and does not say]
\label{rem:deff-scope}
The effective dimension is a ridge-specific quantity. It is not a
drop-in replacement for $\Tr(\Knode)$ in the Rademacher complexity of
the entire ball $\norm{w}_2 \le B$: it arises naturally only for the
ridge algorithm through the smoother $S_{\lreg}$, bias--variance
decompositions, and localized analyses.
\end{remark}

\begin{remark}[Regime for the bias--variance decomposition]
\label{rem:bias-variance-regime}
The conditional-noise hypothesis $\E[\varepsilon \mid x_1^n, \Phinode] = 0$ holds under the honest and cross-fit proof boundaries PB-02 and PB-03 of \cref{sec:proof-boundary}, where the representation does not use the labels of the final fit. Under a label-dependent representation, where the same labels are used both to choose the splits or node weights and to fit the final model, conditioning on $\Phinode$ conditions on a function of those labels, so the hypothesis generally fails. In that regime \cref{eq:ridge-bias-variance} becomes a diagnostic for the fitted representation rather than a population bias--variance decomposition. See \cref{app:honesty-matters} for the parallel statement on the uniform bounds.
\end{remark}

\subsection{Logistic-risk bound and calibration (classification)}
\label{sec:rademacher-classification}

The trace Rademacher bound of \cref{prop:trace-rademacher} applies
verbatim to the homogeneous linear class $\Fcal_B$ of
\cref{eq:fcal-B-linear}: its statement and proof use only
$\norm{\Phinode(x_i)}_2 \le 1$ and are independent of the task-specific
loss. Specialising to the logistic surrogate
$\ell_{\log}(g, y) = \log(1 + \exp(-y g))$, which is 1-Lipschitz in
$g$ and bounded on $\norm{w}_2 \le B$ by
\begin{equation}
  C_B \;:=\; \log(1 + e^{B}),
  \label{eq:logistic-bounding-constant}
\end{equation}
the Lipschitz contraction inequality yields the following uniform
logistic-risk bound.

The empirical KPP-Logistic objective at regularisation $\lreg > 0$ is
\begin{equation}
  \widehat w^{\log}_{\lreg}
  := \arg\min_{w \in \R^m}
     \frac{1}{n}\sum_{i=1}^{n}
     \ell_{\log}\!\bigl(g_w(X_i), Y_i\bigr)
     + \frac{\lreg}{2}\,\norm{w}_2^2
  \;=\;
  \arg\min_{w \in \R^m}
     \Riskhat(g_w) + \frac{\lreg}{2}\,\norm{w}_2^2,
  \label{eq:kpp-logistic-objective}
\end{equation}
with score $g_w(x) = \inner{w}{\Phinode(x)}$ and predicted label
$\sign\bigl(g_{\widehat w^{\log}_{\lreg}}(x)\bigr)$. The objective is
homogeneous ($b = 0$), consistent with the class on which the
classification bounds are stated; when the root weight is positive
($\nodew_t(r_t) > 0$) the constant root coordinate of
\cref{def:tree-embedding} supplies an intercept-like term inside $w$.

\begin{theorem}[Uniform logistic bound]
\label{thm:uniform-logistic-bound}
Condition on the KPP representation in the sense of
\cref{sec:proof-boundary}. For any $B > 0$ and any $\delta \in (0, 1)$,
with probability at least $1 - \delta$, simultaneously for all $w$
with $\norm{w}_2 \le B$,
\begin{equation}
  \Risk(g_w)
  \;\le\;
  \Riskhat(g_w)
  + \frac{2 B}{n}\sqrt{\Tr(\Knode)}
  + 3 C_B \sqrt{\frac{\log(2/\delta)}{2 n}},
  \label{eq:uniform-logistic-bound}
\end{equation}
where $g_w(x) = \inner{w}{\Phinode(x)}$ and $\Risk, \Riskhat$ are the
population and empirical logistic-loss risks of
\cref{eq:risk-emp-risk}.
\end{theorem}

\begin{proof}
Symmetrization bounds the uniform deviation by twice the empirical
Rademacher complexity of the loss class. The Lipschitz contraction
inequality applies because $\ell_{\log}(\cdot, y)$ is $1$-Lipschitz in
the score for every $y \in \set{-1, +1}$, and reduces the loss-class
complexity to the score complexity $\widehat{\Rcal}_n(\Fcal_B)$.
\Cref{prop:trace-rademacher} controls the latter by
$B\sqrt{\Tr(\Knode)}/n$. A bounded-differences step with range $C_B$
gives the confidence term.
\end{proof}

\paragraph{Calibration to classification error.}
For $y \in \set{-1, +1}$ and any score $g$, the $0/1$ loss satisfies
the pointwise inequality (proved in \cref{app:logistic-calibration-proof})
\begin{equation}
  \one\set{y g \le 0}
  \;\le\;
  \frac{\ell_{\log}(g, y)}{\log 2}.
  \label{eq:logistic-pointwise-calibration}
\end{equation}
Therefore, for the classifier $x \mapsto \sign(g_w(x))$, the
population classification error obeys
\begin{equation}
  \Risk_{0/1}(g_w)
  \;:=\;
  \Prob\!\bigl(Y\, g_w(X) \le 0\bigr)
  \;\le\;
  \frac{1}{\log 2}\,\Risk(g_w).
  \label{eq:logistic-calibration}
\end{equation}
Combining \cref{eq:logistic-calibration} with
\cref{thm:uniform-logistic-bound} gives, under the same conditioning
and with probability at least $1 - \delta$, a uniform upper bound on
the classification error in terms of the empirical logistic risk and
the trace term $2B\sqrt{\Tr(\Knode)}/n$.

Sharper excess-risk statements --- in particular fast-rate
classification bounds under low-noise or margin conditions --- can be
imported from standard classification-calibration arguments
\citep{BartlettJordanMcAuliffe2006Calibration,MammenTsybakov1999SmoothDiscrimination};
they are not specific to KPP and require their own assumptions. As
discussed in \cref{sec:robust-radius-classification}, the safe entry
point in this representation is the direct score-margin condition
\cref{ass:direct-score-margin}: a geometric-margin tail condition in
the KPP metric $\deltaT$ does not by itself imply a score-margin tail
condition on $\abs{g_w(X)}$.

\subsection{Partition oracles for one honest tree and forest-level averages}
\label{sec:partition-oracles}

For a single recursive binary partition, the leaf-wise oracle risk
decomposes as a root quantity minus a telescoping sum of local
refinement gains. The argument is partition-theoretic and does not
depend on whether the local quantity is a conditional variance
(regression) or a conditional entropy (classification): in both
cases the same telescope arises from the identity
$p(v) = p(v_L) + p(v_R)$ at every internal node. We isolate the
common skeleton in \cref{lem:partition-telescoping} and instantiate
it in \cref{cor:variance-oracle,cor:entropy-oracle}.

\paragraph{Notation.}
Fix a recursive binary tree $t$ as in
\cref{def:recursive-binary-partition} with regions
$(S_v)_{v \in V_t}$, and write $p(v) := \Prob(X \in S_v)$. Throughout,
$p(v) > 0$ for every node considered; functionals at null-mass nodes are
defined by the convention $p(v)\,\Delta(v) = 0$. For an
internal node $v$ with children $v_L, v_R$, write
\begin{equation}
  \pi_L(v) := \Prob(X \in S_{v_L} \mid X \in S_v),
  \qquad
  \pi_R(v) := \Prob(X \in S_{v_R} \mid X \in S_v),
  \label{eq:child-conditional-probs}
\end{equation}
so $p(v_L) = p(v) \pi_L(v)$ and $p(v_R) = p(v) \pi_R(v)$.

\begin{lemma}[Partition telescoping for a node-level functional]
\label{lem:partition-telescoping}
Let $F : V_t \to \R$ be any node-level functional, and define the
local refinement gain at an internal node $v$ by
\begin{equation}
  \Delta F(v)
  \;:=\;
  F(v) - \pi_L(v) F(v_L) - \pi_R(v) F(v_R).
  \label{eq:partition-gain}
\end{equation}
Then
\begin{equation}
  F(r_t)
  \;=\;
  \sum_{\ell \in V_t^{\mathrm{leaf}}} p(\ell)\, F(\ell)
  \;+\;
  \sum_{v \in V_t^{\mathrm{int}}} p(v)\, \Delta F(v).
  \label{eq:partition-telescoping}
\end{equation}
\end{lemma}

\begin{proof}
For every internal node $v$, multiply \cref{eq:partition-gain} by
$p(v)$ and use $p(v) \pi_L(v) = p(v_L)$, $p(v) \pi_R(v) = p(v_R)$ to
obtain
\begin{equation*}
  p(v)\, F(v)
  \;=\;
  p(v_L)\, F(v_L) + p(v_R)\, F(v_R) + p(v)\, \Delta F(v).
\end{equation*}
Sum this identity over all internal nodes $v \in V_t^{\mathrm{int}}$.
Every non-root node appears exactly once as a child of its parent, so
the internal-node contributions on the right cancel the corresponding
non-root terms on the left. The surviving terms are the root term
$p(r_t)\, F(r_t) = F(r_t)$ and the leaf terms
$\sum_{\ell \in V_t^{\mathrm{leaf}}} p(\ell)\, F(\ell)$.
\end{proof}

\paragraph{Regression: variance oracle.}
Let $\mu(v) := \E[Y \mid X \in S_v]$ and
$Q(v) := \Var(Y \mid X \in S_v)$. By the law of total variance applied
to the binary refinement $S_v = S_{v_L} \cup S_{v_R}$, the variance
decrease at an internal node $v$,
\begin{equation}
  \Delta Q(v)
  \;:=\;
  Q(v) - \pi_L(v) Q(v_L) - \pi_R(v) Q(v_R),
  \label{eq:delta-q}
\end{equation}
is nonnegative. The leaf-wise oracle predictor of the tree is
$f_t^{\mathrm{leaf}}(x) := \mu(\ell_t(x))$, with squared-error risk
\begin{equation}
  R_{\mathrm{leaf}}(t)
  \;:=\;
  \E\bigl[(Y - f_t^{\mathrm{leaf}}(X))^2\bigr]
  \;=\;
  \sum_{\ell \in V_t^{\mathrm{leaf}}} p(\ell)\, Q(\ell).
  \label{eq:leaf-mse}
\end{equation}

\begin{corollary}[Variance oracle for one tree]
\label{cor:variance-oracle}
For any recursive binary tree partition with $S_{r_t} = \Xspace$,
\begin{equation}
  \Var(Y)
  \;=\;
  R_{\mathrm{leaf}}(t) + IG_{\mathrm{var}}(t),
  \qquad
  IG_{\mathrm{var}}(t)
  \;:=\;
  \sum_{v \in V_t^{\mathrm{int}}} p(v)\, \Delta Q(v).
  \label{eq:variance-oracle-tree}
\end{equation}
\end{corollary}

\begin{proof}
Apply \cref{lem:partition-telescoping} with $F(v) = Q(v)$. The root
satisfies $F(r_t) = Q(r_t) = \Var(Y)$ because $S_{r_t} = \Xspace$.
The leaf sum gives $R_{\mathrm{leaf}}(t)$ and the internal sum gives
$IG_{\mathrm{var}}(t)$.
\end{proof}

\paragraph{Classification: entropy oracle.}
Under the convention $y \in \set{-1, +1}$ from \cref{ass:data}, set
\begin{equation}
  \eta(v) \;:=\; \Prob(Y = +1 \mid X \in S_v) \;\in\; [0, 1].
\end{equation}
The conditional distribution of $Y$ given $X \in S_v$ is fully
captured by $\eta(v)$: with the relabelling
$\one\set{Y = +1} \in \set{0, 1}$, one has
$\Prob(\one\set{Y = +1} = 1 \mid X \in S_v) = \eta(v)$, so the entropy
and logistic quantities below depend on $\eta(v)$ only. Let
\begin{equation}
  h(p) \;:=\; - p \log p - (1 - p) \log(1 - p),
  \qquad 0 \le p \le 1,
\end{equation}
with the convention $0 \log 0 = 0$. The entropy decrease at an
internal node,
\begin{equation}
  \Delta h(v)
  \;:=\;
  h(\eta(v)) - \pi_L(v)\, h(\eta(v_L)) - \pi_R(v)\, h(\eta(v_R)),
  \label{eq:delta-h}
\end{equation}
is nonneg by concavity of $h$. The leaf-wise logistic oracle of the
tree is
\begin{equation}
  f_t^{\mathrm{leaf}}(x)
  \;:=\;
  \log \frac{\eta(\ell_t(x))}{1 - \eta(\ell_t(x))},
\end{equation}
with the extended-logit convention when the leaf probability is $0$
or $1$. For $Y \in \set{-1, +1}$ with $\Prob(Y = +1 \mid X \in S_\ell)
= \eta(\ell)$, the conditional logistic loss
$\E[\log(1 + e^{- Y g}) \mid X \in S_\ell]$ is minimized at the
logit value above, with minimum value $h(\eta(\ell))$. Hence the
leaf-wise logistic risk is
\begin{equation}
  L_{\log}^{\mathrm{leaf}}(t)
  \;=\;
  \sum_{\ell \in V_t^{\mathrm{leaf}}} p(\ell)\, h(\eta(\ell)).
  \label{eq:leaf-logistic-risk}
\end{equation}

\begin{corollary}[Entropy oracle for one tree]
\label{cor:entropy-oracle}
For any recursive binary tree partition,
\begin{equation}
  h(\eta(r_t))
  \;=\;
  L_{\log}^{\mathrm{leaf}}(t) + IG_h(t),
  \qquad
  IG_h(t)
  \;:=\;
  \sum_{v \in V_t^{\mathrm{int}}} p(v)\, \Delta h(v).
  \label{eq:entropy-oracle-tree}
\end{equation}
\end{corollary}

\begin{proof}
Apply \cref{lem:partition-telescoping} with $F(v) = h(\eta(v))$. The
leaf sum is $L_{\log}^{\mathrm{leaf}}(t)$ and the internal sum is
$IG_h(t)$. A direct derivation is also given in
\cref{app:entropy-telescoping-proof}.
\end{proof}

\paragraph{Forest-level averages.}
For a finite forest $T$, define the tree-wise averaged oracles
\begin{equation}
  R_{\mathrm{leaf}}(T)
  := \frac{1}{\abs{T}} \sum_{t \in T} R_{\mathrm{leaf}}(t),
  \qquad
  IG_{\mathrm{var}}(T)
  := \frac{1}{\abs{T}} \sum_{t \in T} IG_{\mathrm{var}}(t),
  \label{eq:forest-average-reg}
\end{equation}
\begin{equation}
  L_{\log}^{\mathrm{leaf}}(T)
  := \frac{1}{\abs{T}} \sum_{t \in T} L_{\log}^{\mathrm{leaf}}(t),
  \qquad
  IG_h(T)
  := \frac{1}{\abs{T}} \sum_{t \in T} IG_h(t).
  \label{eq:forest-average-clf}
\end{equation}
The tree-level identities of
\cref{cor:variance-oracle,cor:entropy-oracle} average to
\begin{equation}
  R_{\mathrm{leaf}}(T) = \Var(Y) - IG_{\mathrm{var}}(T),
  \qquad
  L_{\log}^{\mathrm{leaf}}(T) = h(\eta(r)) - IG_h(T),
  \label{eq:forest-average-oracle-identity}
\end{equation}
where $r$ denotes the common root event $\Xspace$. These are averages
over tree-specific leaf oracles. They are not the oracle on the
common refinement of all tree leaves, and they are stable under
duplicating identical trees.

\begin{remark}[Why raw forest sums are wrong]
\label{rem:raw-forest-sum-pitfall}
A raw sum
$\sum_{t \in T} \sum_{v \in V_t^{\mathrm{int}}} p(v)\, \Delta F(v)$
(with $F = Q$ for variance, or $F = h \circ \eta$ for entropy) is
\emph{not} the refinement gain of a single forest partition unless an
explicit common-refinement partition and its telescoping tree have
been constructed. Duplicating the same tree $K$ times multiplies the
raw sum by $K$, while the leaf-wise risk of the (unchanged) partition
does not change. The tree-wise averages
$IG_{\mathrm{var}}(T)$ and $IG_h(T)$ in
\cref{eq:forest-average-oracle-identity} are the correct
forest-level invariants.
\end{remark}

\paragraph{Empirical estimation and honesty.}
In practice, $p(v)$, $\mu(v)$, $Q(v)$, $\eta(v)$, and the local gains
$\Delta Q(v), \Delta h(v)$ are estimated from data. The honest
construction of \cref{sec:proof-boundary} separates two tasks: a
partition fold builds the trees and estimates the local gains, and a
disjoint fit fold builds the KPP design $\Phinode$ and fits the final
ERM analysed in
\cref{sec:rademacher-regression,sec:rademacher-classification}. Under
this separation, the empirical refinement gains $\widehat{IG}_{\mathrm{var}}(T)$
and $\widehat{IG}_h(T)$ are partition-quality diagnostics that are
not directly contaminated by the labels used to fit $\widehat w$.
They remain estimates of population quantities; finite-sample
concentration for $\widehat{IG}_{\mathrm{var}}(T)$ and
$\widehat{IG}_h(T)$ should be added separately if it is needed as a
theorem.

\section{Experiments}
\label{sec:experiments}

\subsection{Benchmark on five tabular datasets}
\label{sec:bench-five-datasets}

The first empirical objective is not to claim state-of-the-art
performance, but to test whether the KPP representation and its
companion dashboard (\cref{sec:dashboard-conditional}) make the
predictive behaviour of the method legible alongside conventional
tree-ensemble baselines. We report predictive metrics on five
tabular datasets that span the regression / classification axis and
a range of $(n, p)$ regimes.

\paragraph{Datasets.}
\Cref{tab:bench-datasets} lists the five datasets used in the
benchmark, with sample sizes and feature counts. The
$80\%/20\%$ train/test split is the same across all methods, and
each dataset is run with five independent seeds for the train/test
shuffle. Class balance and noise levels are taken as given by the
source repositories; no class re-balancing or feature transformation
is applied beyond standard scaling for the raw-feature baselines.

\begin{table}[t]
  \centering
  \small
  \caption{Datasets used in the benchmark of
    \cref{sec:bench-five-datasets}. Primary metric is RMSE for
    regression and the test error rate $1 - \mathrm{accuracy}$ for
    classification, both lower-is-better. Five seeds per dataset.}
  \label{tab:bench-datasets}
  \begin{tabular}{lllrrr}
    \toprule
    Dataset & Task & Source & $n_{\mathrm{train}}$ & $n_{\mathrm{test}}$ & $p$ \\
    \midrule
    \texttt{breast\_cancer} & classification & sklearn & 455 & 114 & 30 \\
    \texttt{california\_housing\_subsample\_2000} & regression & sklearn & 1\,600 & 400 & 8 \\
    \texttt{concrete} & regression & OpenML & 824 & 206 & 8 \\
    \texttt{spambase\_full} & classification & OpenML & 3\,680 & 921 & 57 \\
    \texttt{wine\_quality\_red} & regression & OpenML & 1\,279 & 320 & 11 \\
    \bottomrule
  \end{tabular}
\end{table}

\paragraph{Baselines.}
Under identical splits, the KPP estimator
($\widehat f$ in \cref{sec:setup-risk}, either KPP-Ridge for
regression or KPP-Logistic for classification) is compared against
six baselines that cover the standard tabular landscape:
\begin{itemize}
  \item Random Forest regression / classification
    \citep{Breiman2001RandomForests};
  \item Gradient Boosting trees as implemented in scikit-learn
    \citep{Friedman2001GBM};
  \item XGBoost and LightGBM
    \citep{ChenGuestrin2016XGBoost,KeEtAl2017LightGBM};
  \item ridge regression (resp.\ logistic regression) on the raw
    features.
\end{itemize}
A more recent gradient-boosted baseline, CatBoost
\citep{Prokhorenkova2018CatBoost}, is reported factually in the
per-dataset tables without a methodological discussion: it is
included to anchor KPP relative to the current production stack,
not to certify a comparison.

\paragraph{Leaf-only ablation, and baselines still deferred.}
The leaf-only variant of the KPP representation --- the same forest and
fit pipeline, but with the internal-node path coordinates removed so that
only the tree-leaf indicators remain --- is no longer deferred: we run it
as an internal ablation against the full KPP estimator (the \texttt{KPP}
row of the per-dataset tables, which reproduces the benchmark KPP metric
bit for bit) on every dataset of the 40-dataset suite of
\cref{sec:bench-demsar}, reusing each seed's tuned $\lreg$ so that the two
arms differ only in the representation. Averaging the primary metric over
seeds per dataset and comparing the two arms head to head, the
internal-node path coordinates help on a majority of the suite ---
$23$ of the $40$ datasets favour the full estimator, $15$ favour
leaf-only, and on $2$ the two arms are indistinguishable up to
numerical precision --- but the advantage is \emph{not}
uniform: leaf-only is competitive with, or better than, the full estimator
on a substantial minority, most markedly on \texttt{auto\_price}, where it
lowers RMSE relative to the full estimator by about $35.8$. We read this as
evidence that the full-path and leaf-only representations are two members
of one family whose better arm is dataset-dependent, rather than as a
blanket superiority of either; it is an internal ablation of the KPP
representation and so sits on a different footing from the cross-method
ranking of \cref{sec:bench-demsar}. The five per-dataset tables of the
\hyperref[suppl:bench-detailed]{Supplementary Material} report the leaf-only metric as a dedicated
\texttt{KPP\_leaf-only} row. A precise characterisation of \emph{which}
structural features make the internal-node path coordinates pay off is
left to future work, as is a comparison against an Extra-Trees ensemble
\citep{GeurtsErnstWehenkel2006ExtraTrees}.

\paragraph{Ablations.}
The minimal ablation grid varies the node weights (variance decrease,
raw variance, or uniform), the tree depth and minimum leaf size, the
number of trees $K$, the regularisation parameter $\lreg$, and
whether the representation is honestly split from the fit fold (when
the dataset size permits). The grid is the same across the two
tasks; only the ERM loss differs. Of these axes, the node-weight choice
is the one we evaluate empirically here: \cref{tab:node-weight} in
\cref{app:variance-gain-choice} compares the three weightings head to head
across the 40-dataset suite, finding the variance-gain default robust but
not universally optimal.

\paragraph{Reporting recommendation.}
As a practitioner recommendation for the KPP dashboard, we suggest that
an application report, for each dataset and configuration, the predictive
metric together with the diagnostics. For regression:
RMSE / MAE / $R^2$, $\Tr(\Knode)/n$, $\deff(\lreg)/n$ from
\cref{eq:effective-dimension}, $\mueig_{\min}(\Knode)$, the
half-forest discrepancy $\Deltak$ of
\cref{rem:operator-norm-diagnostic}, the empirical variance gain
$\widehat{IG}_{\mathrm{var}}(T)$ from \cref{cor:variance-oracle},
and $\norm{\widehat w_{\lreg}}_2$. For classification: error rate /
balanced accuracy / AUC where appropriate, calibration metrics, and
the same dashboard quantities with $\widehat{IG}_{\mathrm{var}}(T)$
replaced by the entropy gain $\widehat{IG}_h(T)$ from
\cref{cor:entropy-oracle}. The empirical question such a protocol
targets is whether the diagnostics predict failure modes: high
capacity, poor conditioning, unstable forests, weak partitions, or
overly large Lipschitz constants. The present paper does not run this
full protocol on every dataset: on the five benchmark datasets it
reports the primary metric, mean fit time, and seed count
(\cref{tab:bench-rank} and
\cref{tab:bench-breast-cancer,tab:bench-california,tab:bench-concrete,tab:bench-spambase,tab:bench-wine})
together with the
dashboard observables of \cref{tab:dashboard-empirical} ($\lreg$,
$\norm{\widehat w}_2$, $\Tr(\Knode)/n$, $\deff(\lreg)/n$,
$\mueig_{\min}+\lreg$, $\Deltak$, and $\widehat{IG}_\star$) and the
uniform bounds of \cref{tab:dashboard-bounds}. The dashboard analysis
itself is deferred to \cref{sec:dashboard-conditional}.

\paragraph{Bench results.}
\Cref{tab:bench-rank} reports the primary metric of the KPP estimator
on each dataset, together with the best baseline metric and the
resulting rank of KPP among the six baselines plus KPP itself. KPP
attains rank $1$ on \texttt{spambase\_full} and \texttt{wine\_quality\_red}.
On \texttt{concrete} it sits at rank $2$, within one standard
deviation of LightGBM. On \texttt{california\_housing\_subsample\_2000}
it sits at rank $4$, roughly $5\%$ above CatBoost in RMSE. On
\texttt{breast\_cancer} it sits at rank $3$, behind logistic
regression on the raw features and CatBoost, on a dataset where a
linear baseline on 30 features is already strong. Per-dataset
detailed tables (all six baselines, fit-time and per-seed
variability) are deferred to the
\hyperref[suppl:bench-detailed]{Supplementary Material}.

\begin{table}[t]
  \centering
  \small
  \caption{Primary metric of KPP and best baseline per dataset on the
    benchmark of \cref{sec:bench-five-datasets}. Regression metric is
    RMSE; classification metric is the test error rate
    $1 - \mathrm{accuracy}$, both lower-is-better. Five seeds, mean
    $\pm$ std. Rank counts KPP among six baselines plus itself; bold
    entries are rank $1$. The KPP entries are taken from the post-fit
    estimator $\widehat f$ under the label-dependent proof-boundary
    regime PB-04 (the KPP representation is built and the final model
    is fit on the same training labels; see \cref{sec:proof-boundary}).}
  \label{tab:bench-rank}
  \begin{tabular}{llrlrr}
    \toprule
    Dataset & Task & KPP & Best baseline & Best metric & KPP rank \\
    \midrule
    \texttt{breast\_cancer} & class. & $0.0281 \pm 0.0039$ & LogReg\_raw & $0.0193 \pm 0.0073$ & $3 / 7$ \\
    \texttt{california\_housing\_sub.\_2000} & reg. & $0.5469 \pm 0.0413$ & CatBoost & $0.5183 \pm 0.0379$ & $4 / 7$ \\
    \texttt{concrete} & reg. & $4.3024 \pm 0.2299$ & LightGBM & $4.2472 \pm 0.2464$ & $2 / 7$ \\
    \texttt{spambase\_full} & class. & $\mathbf{0.0452 \pm 0.0033}$ & LightGBM & $0.0456 \pm 0.0051$ & $\mathbf{1 / 7}$ \\
    \texttt{wine\_quality\_red} & reg. & $\mathbf{0.5506 \pm 0.0272}$ & RandomForest & $0.5577 \pm 0.0279$ & $\mathbf{1 / 7}$ \\
    \bottomrule
  \end{tabular}
\end{table}

The KPP estimator is competitive in predictive metrics on this
five-dataset slice without dominating it. Its differentiating value
relative to the gradient-boosted baselines lies in the structural
guarantees of the preceding sections: exact additive node-level
attribution and, by marginal split, an exact variable-level decomposition
(\cref{prop:attribution-node,prop:attribution-variable}), the
deterministic Lipschitz robust-radius certificate
(\cref{prop:lipschitz-kpp,prop:classification-robust-radius}), the
trace-based Rademacher and effective-dimension diagnostics
(\cref{prop:trace-rademacher,eq:effective-dimension}), and the
partition oracles of
\cref{cor:variance-oracle,cor:entropy-oracle}. None of these are
available out of the box from the boosted baselines.
Per-dataset breakdown tables with detailed results are provided in
\cref{tab:bench-breast-cancer,tab:bench-california,tab:bench-concrete,tab:bench-spambase,tab:bench-wine}.

\subsection{Aggregate ranking across a 40-dataset suite (Friedman--Nemenyi)}
\label{sec:bench-demsar}

The five-dataset slice of \cref{sec:bench-five-datasets} is too small to
support a statistical statement about how KPP ranks against the baselines
in aggregate. We therefore extend the comparison to a broader suite of
$40$ datasets --- $20$ regression and $20$ classification tasks drawn from
the OpenML and scikit-learn collections --- run under the same protocol
($80\%/20\%$ split, five seeds, standard scaling for the raw-feature
baselines) for the same seven methods: KPP and the six baselines of
\cref{sec:bench-five-datasets} (Random Forest, Gradient Boosting, XGBoost,
LightGBM, CatBoost, and ridge / logistic regression on the raw features).

\paragraph{Protocol.}
We follow the standard methodology for comparing multiple methods over
multiple datasets \citep{demsar2006}: a Friedman omnibus test on the
per-dataset ranks and, when the omnibus null is rejected, the Nemenyi
post-hoc test, summarised by a critical-difference (CD) diagram. The two
tasks are analysed \emph{separately} and never pooled: the regression
metric (RMSE) and the classification metric (the error rate
$1 - \mathrm{accuracy}$) are not commensurable and cannot be co-ranked.
Within each task both metrics are lower-is-better, so on a given dataset
rank $1$ is the method with the lowest metric.

\paragraph{Aggregate results.}
On both tasks the tie-corrected Friedman test (as implemented in
\texttt{scipy.stats.friedmanchisquare}) rejects the null of equal mean
ranks: for regression $\chi^2 = 37.95$ ($p = 1.15 \times 10^{-6}$, $N = 20$),
and for classification $\chi^2 = 21.22$ ($p = 1.68 \times 10^{-3}$, $N = 20$),
each over $k = 7$ methods; the regression ranks contain no ties, so the
tie correction is the identity there. The Nemenyi critical difference at $\alpha = 0.05$ is
$\mathrm{CD} = 2.014$ ($q_{0.05} = 2.948$). \Cref{tab:bench-demsar} lists
the mean ranks per method and the post-hoc status relative to KPP, and
\cref{fig:cd-regression,fig:cd-classification} show the corresponding CD
diagrams. KPP attains the best mean rank in regression ($2.30$) and the
second-best in classification ($3.075$, behind CatBoost at $2.675$). Under
the Nemenyi post-hoc test, KPP is in statistical parity with the strongest
baselines on both tasks: in regression it is indistinguishable from
CatBoost, Random Forest, and Gradient Boosting, and separated only from
LightGBM, XGBoost, and the raw-feature ridge baseline; in classification it
is indistinguishable from every baseline except Gradient Boosting. We read
these as parity results, not as evidence of dominance: where the mean-rank
gap exceeds the critical difference it happens to favour KPP, but the
headline of the aggregate comparison is that KPP sits among the top
tree-ensemble methods with no statistically resolvable separation from the
best of them.

\begin{table}[t]
  \centering
  \small
  \caption{Mean ranks of KPP and the six baselines across the 40-dataset
    suite (\cref{sec:bench-demsar}), computed per task and never pooled:
    regression over $N = 20$ datasets (RMSE) and classification over
    $N = 20$ datasets (error rate $1 - \mathrm{accuracy}$), both
    lower-is-better, so rank $1$ is the lowest metric. ``Nemenyi vs.\ KPP''
    is the post-hoc status relative to KPP at $\alpha = 0.05$ (critical
    difference $\mathrm{CD} = 2.014$, $q_{0.05} = 2.948$, $k = 7$):
    \emph{parity} when the mean-rank gap to KPP does not exceed the CD,
    \emph{separated} otherwise; in every separated case here KPP holds the
    lower rank. Methods are ordered by mean rank within each panel.}
  \label{tab:bench-demsar}
  \begin{tabular}{lcc}
    \toprule
    \multicolumn{3}{l}{\emph{(a) Regression --- $N = 20$ datasets (RMSE)}} \\
    Method & Mean rank & Nemenyi vs.\ KPP \\
    \midrule
    KPP              & $2.30$ & ---       \\
    CatBoost         & $2.75$ & parity    \\
    RandomForest     & $3.80$ & parity    \\
    GradientBoosting & $3.90$ & parity    \\
    LightGBM         & $4.55$ & separated \\
    XGBoost          & $4.90$ & separated \\
    Ridge\_raw       & $5.80$ & separated \\
    \midrule
    \multicolumn{3}{l}{\emph{(b) Classification --- $N = 20$ datasets ($1 - \mathrm{accuracy}$)}} \\
    Method & Mean rank & Nemenyi vs.\ KPP \\
    \midrule
    CatBoost         & $2.675$ & parity    \\
    KPP              & $3.075$ & ---       \\
    RandomForest     & $3.700$ & parity    \\
    LogReg\_raw      & $4.225$ & parity    \\
    LightGBM         & $4.425$ & parity    \\
    XGBoost          & $4.675$ & parity    \\
    GradientBoosting & $5.225$ & separated \\
    \bottomrule
  \end{tabular}
\end{table}

\begin{figure}[t]
  \centering
  \includegraphics[width=\textwidth]{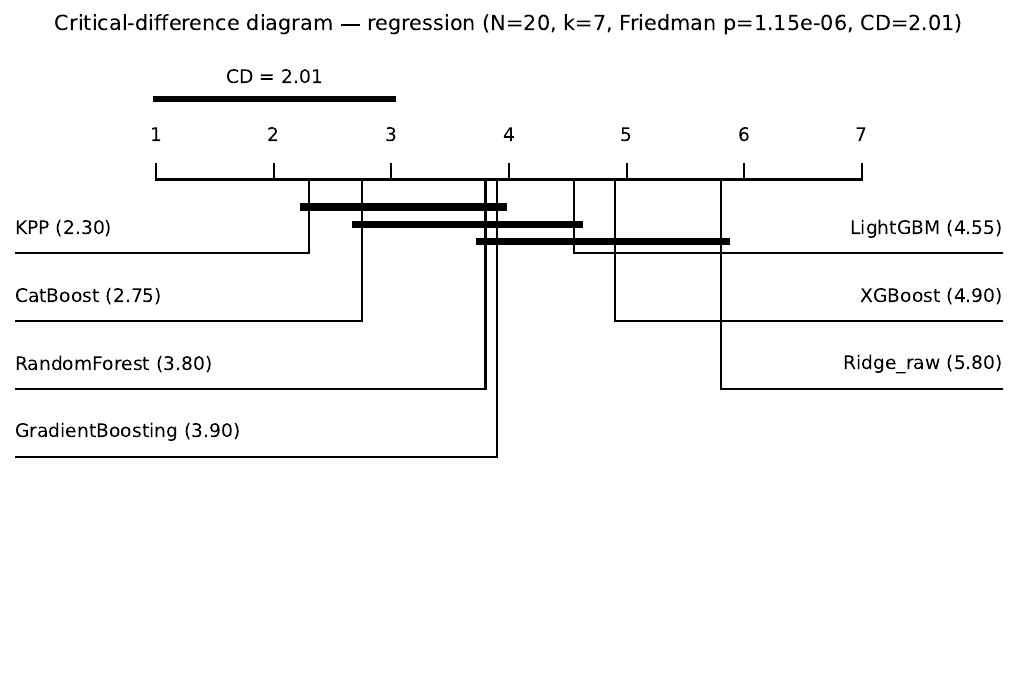}
  \caption{Critical-difference (CD) diagram for the regression task of the
    40-dataset suite (\cref{sec:bench-demsar}): each of the seven methods is
    placed at its mean rank over $N = 20$ datasets, and methods whose
    mean-rank gap does not exceed the Nemenyi critical difference
    ($\mathrm{CD} = 2.014$, $\alpha = 0.05$) are joined by a bar. Rank $1$
    is the best (lowest-metric) method.}
  \label{fig:cd-regression}
\end{figure}

\begin{figure}[t]
  \centering
  \includegraphics[width=\textwidth]{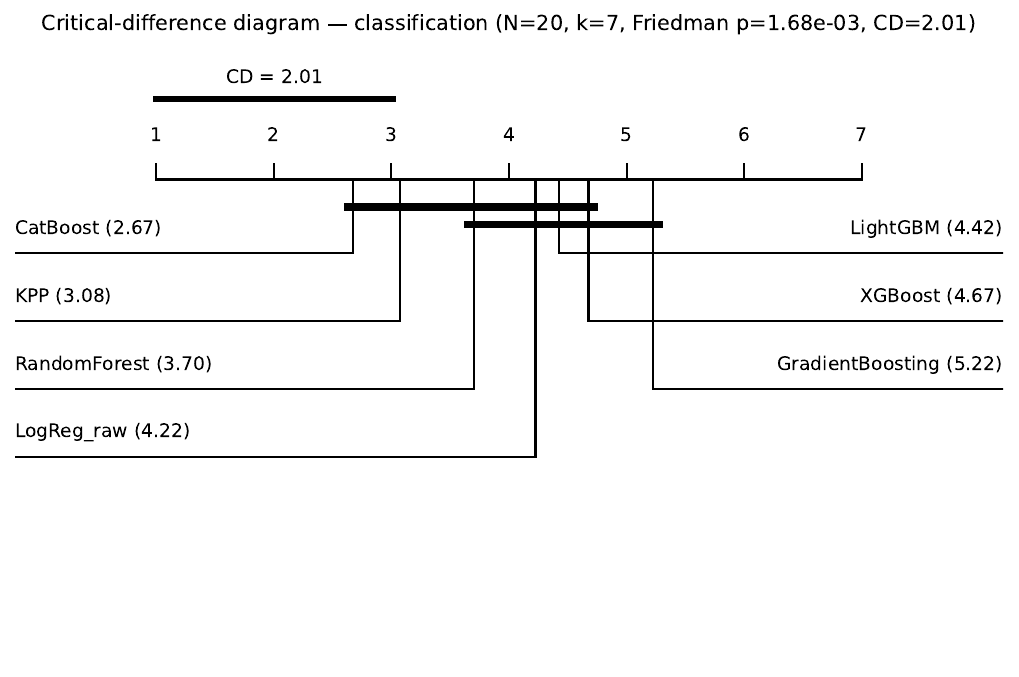}
  \caption{Critical-difference (CD) diagram for the classification task of
    the 40-dataset suite (\cref{sec:bench-demsar}): each of the seven
    methods is placed at its mean rank over $N = 20$ datasets, and methods
    whose mean-rank gap does not exceed the Nemenyi critical difference
    ($\mathrm{CD} = 2.014$, $\alpha = 0.05$) are joined by a bar. Rank $1$
    is the best (lowest-metric) method.}
  \label{fig:cd-classification}
\end{figure}

This aggregate picture is complementary to, not in tension with, the
per-dataset reading of \cref{sec:bench-five-datasets}: attaining rank $1$
on \texttt{spambase\_full} and \texttt{wine\_quality\_red} is a statement
about two individual datasets, whereas the parity above is a statement
about mean ranks, computed separately within each task over its $20$ datasets. The two operate at different granularities and
are reported as such.

\subsection{Conditional dashboard and robust radius}
\label{sec:dashboard-conditional}

The KPP dashboard reports five families of observables that should
be read together rather than in isolation, because they monitor
distinct failure modes of the estimator: capacity, conditioning,
finite-forest stability, partition quality, and margin / robustness.
Each family is theorem-backed by an earlier section under the proof
boundary of \cref{sec:proof-boundary}. Empirical values on the
five-dataset benchmark of \cref{sec:bench-five-datasets} are reported
in
\cref{tab:dashboard-empirical,tab:dashboard-bounds,tab:robust-radius}.

\paragraph{Capacity.}
The normalised trace $\Tr(\Knode) / n$ is the average squared norm
of the KPP feature vectors and the quantity that enters the trace
Rademacher bound of \cref{prop:trace-rademacher}. For ridge
regression, a more relevant algorithm-specific quantity is the
normalised effective dimension $\deff(\lreg) / n$ of
\cref{eq:effective-dimension}; under the ridge bias--variance
decomposition of \cref{prop:ridge-bias-variance}, $\deff(\lreg)$
controls the variance term. If $\deff(\lreg) / n$ is large and
validation error is unstable, increasing $\lreg$, reducing tree
depth, clipping extreme node weights, or simplifying the forest are
the natural actions.

\paragraph{Conditioning.}
The smallest eigenvalue $\mueig_{\min}(\Knode)$ is a warning signal
for collinearity but should not be over-interpreted: in
high-dimensional or sparse regimes $\Knode$ may be singular and
ridge stabilises the relevant system through $\Knode + \lreg I_n$. A
safer conditioning proxy is therefore the shifted quantity
$\mueig_{\min}(\Knode) + \lreg$. When the half-forest diagnostic is
also reported, practitioners sometimes inspect
$(\mueig_{\min}(\Knode) - \Deltak)_+ + \lreg$; this is reasonable as
a guardrail but should be described as a diagnostic unless the
comparison between $\Deltak$ and $\norm{\Knode - \Ginf}_\op$ has
been calibrated, as discussed in
\cref{rem:operator-norm-diagnostic}.

\paragraph{Finite-forest stability.}
Building two independent half-forests and computing
$\Deltak = \norm{\GK^{(1)} - \GK^{(2)}}_\op / \sqrt{2}$ provides a
practical stopping signal: a decreasing curve $K \mapsto \Deltak$
suggests that additional trees are still stabilising the
representation, and a plateau suggests that adding trees is unlikely
to change the geometry much. This is a stopping diagnostic, not a
generalisation theorem; cf.\ \cref{rem:operator-norm-diagnostic}.

\paragraph{Partition quality.}
The empirical variance gain $\widehat{IG}_{\mathrm{var}}(T)$ of
\cref{cor:variance-oracle} and the empirical entropy gain
$\widehat{IG}_h(T)$ of \cref{cor:entropy-oracle} are partition-
quality diagnostics under the honest split of
\cref{sec:proof-boundary}. If they remain small even for deeper or
more diverse forests, the bottleneck is the partition mechanism
rather than the final linear layer.

\paragraph{Margins (classification).}
For classification, the empirical signed margins
$m_i := y_i\, g_{\widehat w}(x_i)$ are reported together with the
empirical CDFs of $m_i$ and $\abs{g_{\widehat w}(x_i)}$ near zero,
and with the robust-accuracy curve of
\cref{eq:robust-accuracy-curve}. Large mass near zero indicates
weak margin and weak local robustness, even when test accuracy is
high.

\paragraph{Dashboard reading guide.}
\Cref{tab:dashboard-summary} compiles the five families into a
unified reading guide: each observable, the symptom it warns about,
and the typical mitigating action.

\begin{table}[t]
  \centering
  \small
  \caption{Dashboard reading guide for the KPP estimator. Common
    rows apply to both tasks; task-specific rows are flagged in the
    Observable column. The qualifiers "large" and "small" are
    qualitative; calibrated thresholds depend on the dataset and on
    the comparator baseline.}
  \label{tab:dashboard-summary}
  \begin{tabular}{p{0.30\textwidth}p{0.33\textwidth}p{0.30\textwidth}}
    \toprule
    Observable & Symptom & Typical action \\
    \midrule
    $\Tr(\Knode) / n$ & Large raw capacity & Reduce depth, clip node weights \\
    $\deff(\lreg) / n$ (reg.) & Many active ridge degrees of freedom & Increase $\lreg$ or simplify the forest \\
    $\mueig_{\min}(\Knode) + \lreg$ & Poor conditioning if small & Increase $\lreg$ \\
    $\Deltak$ & Geometry still unstable in $K$ & Add trees if validation improves; otherwise stop \\
    $\widehat{IG}_{\mathrm{var}}(T)$ (reg.) & Weak partition gain & Revisit splits, depth, or features \\
    $\widehat{IG}_h(T)$ (class.) & Weak partition gain & Revisit splits, depth, or features \\
    Margin CDF near zero (class.) & Weak separation & Increase $\lreg$, tune representation, inspect labels \\
    $\norm{\widehat w}_2$ & Weak robustness if large & Increase $\lreg$, simplify representation \\
    \bottomrule
  \end{tabular}
\end{table}

\paragraph{Empirical dashboard observables.}
\Cref{tab:dashboard-empirical} reports the dashboard observables on
the five benchmark datasets. The entries are five-seed mean $\pm$
standard deviation; entries marked "---" do not apply (the
effective dimension is regression-specific). The empirical
partition-quality gain is $\widehat{IG}_{\mathrm{var}}(T)$ for
regression and $\widehat{IG}_h(T)$ for classification, reported in
the same column with the natural-units convention of each gain
(units of the response variance vs.\ entropy in nats).

\begin{table}[t]
  \centering
  \scriptsize
  \caption{Empirical dashboard observables for the KPP estimator on
    the five benchmark datasets. Five seeds, mean $\pm$ std.
    Conditionally on the learned KPP representation $\Phinode$ (the
    forest is fixed after the training-fold fit, in the sense of
    PB-01 in \cref{sec:proof-boundary}). $\widehat{IG}_\star$ denotes
    $\widehat{IG}_{\mathrm{var}}(T)$ for regression and
    $\widehat{IG}_h(T)$ for classification; end-to-end this full-sample
    regime is the label-dependent regime PB-04
    (\cref{sec:diabetes-honest-crossfit}), so the tabulated quantities
    are conditional diagnostics, not end-to-end guarantees.}
  \label{tab:dashboard-empirical}
  \resizebox{\textwidth}{!}{%
  \begin{tabular}{lrrrrrrr}
    \toprule
    Dataset & $\lreg$ & $\norm{\widehat w}_2$ & $\Tr(\Knode)/n$ & $\deff(\lreg)/n$ & $\mueig_{\min}+\lreg$ & $\Deltak$ & $\widehat{IG}_\star$ \\
    \midrule
    \texttt{breast\_cancer} & $1.00\mathrm{e}{-10}$ & $397.92 \pm 15.54$ & $0.0542 \pm 0.0033$ & --- & $8.02\mathrm{e}{-4} \pm 1.15\mathrm{e}{-4}$ & $0.149 \pm 0.043$ & $0.582 \pm 0.024$ \\
    \texttt{california\_h.\_2000} & $2.34\mathrm{e}{-3}$ & $193.42 \pm 3.38$ & $0.0125 \pm 0.0004$ & $0.625 \pm 0.009$ & $2.72\mathrm{e}{-3} \pm 8.4\mathrm{e}{-5}$ & $0.0565 \pm 0.0111$ & $1.245 \pm 0.058$ \\
    \texttt{concrete} & $2.07\mathrm{e}{-4}$ & $1437.73 \pm 10.17$ & $0.0406 \pm 0.0013$ & $0.933 \pm 0.011$ & $2.07\mathrm{e}{-4} \pm 0$ & $0.112 \pm 0.023$ & $266.72 \pm 21.05$ \\
    \texttt{spambase\_full} & $1.44\mathrm{e}{-7}$ & $759.30 \pm 4.12$ & $0.0057 \pm 0.0001$ & --- & $1.44\mathrm{e}{-7} \pm 0$ & $0.0917 \pm 0.0283$ & $0.575 \pm 0.002$ \\
    \texttt{wine\_quality\_red} & $7.85\mathrm{e}{-3}$ & $118.11 \pm 0.95$ & $0.0101 \pm 0.0002$ & $0.425 \pm 0.004$ & $7.85\mathrm{e}{-3} \pm 0$ & $0.0246 \pm 0.0063$ & $0.478 \pm 0.032$ \\
    \bottomrule
  \end{tabular}%
  }
\end{table}

\paragraph{Bounds RG-02 / RG-03 (regression) and CL-02 / CL-03 (classification).}
RG-02 and CL-02 denote the regression and classification trace
Rademacher bounds obtained from \cref{prop:trace-rademacher}, the
classification version entering through the Lipschitz contraction
step in the proof of \cref{thm:uniform-logistic-bound}, while RG-03
is the uniform squared-loss bound of
\cref{prop:uniform-squared-loss-bound}
(\cref{eq:uniform-squared-loss-bound}) and CL-03 the uniform logistic
bound of \cref{thm:uniform-logistic-bound}
(\cref{eq:uniform-logistic-bound}). The trace Rademacher bound and the
uniform squared-loss / logistic bounds specialised at
$B = \norm{\widehat w}_2$ produce four numerical diagnostics per
dataset. These are conditional bounds evaluated at the realised fit: the radius $B$ is taken from the fitted estimator
$\widehat w$, not over a pre-specified hypothesis class. They
should not be confused with uniform generalisation bounds over the
ball $\set{g : \norm{w}_2 \le B}$, for which $B$ would be fixed
\emph{before} the fit. A bound is reported as \emph{trivial} when it
exceeds the loss of the trivial constant predictor (the null
baseline) it is meant to improve on:
$M^2 = (\sup_i \abs{y_i})^2$ for regression (an upper bound on the
squared loss of the null predictor $\widehat y \equiv 0$), and $\log 2$
for classification (the logistic loss of the constant-$1/2$ predictor); the count of seeds for which the bound is
trivial is reported in
\cref{tab:dashboard-bounds}. Here $\sup_i \abs{y_i}$ is the in-sample plug-in for the a-priori almost-sure envelope $M$ of \cref{prop:uniform-squared-loss-bound}, not itself an almost-sure bound on $Y$; any bound or threshold evaluated with it is therefore an empirical surrogate, in the same sense as the post-hoc radius $B = \norm{\widehat w}_2$.

\begin{table}[t]
  \centering
  \small
  \caption{Conditional bounds for the KPP estimator at
    $B = \norm{\widehat w}_2$. Five seeds, mean $\pm$ std.
    "Trace" denotes the symmetrised trace term
    $2 B \sqrt{\Tr(\Knode)} / n$ (twice the empirical Rademacher
    complexity of \cref{prop:trace-rademacher}, from the symmetrisation
    step in \cref{eq:uniform-logistic-bound}). The logistic bound
    \cref{eq:uniform-logistic-bound} uses this term directly, whereas the
    regression bound \cref{eq:uniform-squared-loss-bound} scales it by the
    loss factor $2(M+B)$; the column reports the loss-independent
    symmetrised Rademacher term common to both;
    "Uniform" denotes the full uniform bound RG-03
    (regression, \cref{eq:uniform-squared-loss-bound} at
    $\delta = 0.05$) or CL-03 (classification,
    \cref{eq:uniform-logistic-bound} at $\delta = 0.05$). The
    "Trivial" column reports the number of seeds (out of five) for
    which the uniform bound exceeds the trivial-constant-predictor loss ($M^2$ for
    regression, $\log 2 \approx 0.693$ for classification); end-to-end
    this full-sample regime is the label-dependent regime PB-04
    (\cref{sec:diabetes-honest-crossfit}), so the tabulated quantities
    are conditional diagnostics, not end-to-end guarantees. The uniform
    bounds RG-03 / CL-03 are stated for the homogeneous class $\Fcal_B$,
    whereas the reference implementation fits a separate unpenalised
    intercept outside $\Fcal_B$ (\cref{sec:setup-risk}); the tabulated
    bounds accordingly pertain to the homogeneous predictor, a
    conditional status assumed here.}
  \label{tab:dashboard-bounds}
  \begin{tabular}{lrrr}
    \toprule
    Dataset & Trace & Uniform & Trivial (of 5) \\
    \midrule
    \texttt{breast\_cancer} & $8.66 \pm 0.13$ & $169.17 \pm 6.13$ & $5$ \\
    \texttt{california\_h.\_2000} & $1.0797 \pm 0.0081$ & $16902.29 \pm 560.99$ & $5$ \\
    \texttt{concrete} & $20.18 \pm 0.27$ & $1\,434\,731 \pm 18\,532$ & $5$ \\
    \texttt{spambase\_full} & $1.889 \pm 0.010$ & $105.72 \pm 0.56$ & $5$ \\
    \texttt{wine\_quality\_red} & $0.663 \pm 0.003$ & $7579.19 \pm 111.05$ & $5$ \\
    \bottomrule
  \end{tabular}
\end{table}

\Cref{tab:dashboard-bounds} confirms that on this benchmark the
uniform bounds, evaluated at the realised $B = \norm{\widehat w}_2$,
are trivial in every seed: the radius $B$ produced by the
unconstrained ridge / logistic fit is large enough that the
trace term times $B$ saturates the trivial-constant-predictor loss. The trace
term itself --- which scales as $B \sqrt{\Tr(\Knode)} / n$ ---
remains a useful comparative diagnostic across datasets and
configurations, but should not be reported as a non-vacuous uniform
generalisation bound at the fitted $B$. This is the
methodological asymmetry that Subsection~7.1 will frame radically.

\paragraph{Robust radius --- empirical certificates.}
\Cref{tab:robust-radius} reports the deterministic robust-radius
certificates of \cref{eq:robust-mse,eq:robust-accuracy-curve} on the
same five datasets. For regression, the certified observable is
$\mathrm{RobMSE}(u)$ from \cref{eq:robust-mse} at three metric radii
$u \in \set{0, u_{\max}/10, u_{\max}}$, where
$u_{\max} := (M / L)^2$ with $L = \norm{\widehat w_{\lreg}}_2$ and
$M = \sup_i \abs{y_i}$. For classification, the certified observable
is $\mathrm{RobAcc}(u)$ from \cref{eq:robust-accuracy-curve} at the
analogous radii with $u_{\max} := (\max_i \abs{m_i} / L)^2$. The
column $u_{\mathrm{half}}$ records the smallest radius at which the
certified metric crosses $50\%$ of its baseline: $\mathrm{RobAcc}$
crossing $50\%$ for classification, $\mathrm{RobMSE}$ doubling for
regression. These are certificates in the KPP path metric
$\deltaT$, not in the raw input metric.

\begin{table}[t]
  \centering
  \scriptsize
  \caption{Empirical robust-radius certificates for the KPP
    estimator. For classification, the certified observable is
    $\mathrm{RobAcc}(u)$; for regression, it is $\mathrm{RobMSE}(u)$.
    Five seeds, mean $\pm$ std. The $u_{\max}$ column reports the
    largest radius plotted, $u_{\max} := (M/L)^2$ with $M = \sup_i \abs{y_i}$
    for regression and $u_{\max} := (\max_i \abs{m_i}/L)^2$ for
    classification, at which the worst-case margin budget is exhausted (for
    regression the bound continues to grow beyond it); $u_{\mathrm{half}}$ is the
    smallest radius at which $\mathrm{RobAcc}$ crosses $50\%$ of its baseline
    (classification) or $\mathrm{RobMSE}$ doubles (regression)
    (cf.\ \cref{eq:robust-accuracy-curve,eq:robust-mse}).
    Certificates are deterministic in $\deltaT$, not in any raw
    input-space metric. All reported robust-radius curves are evaluated on the held-out test set.}
  \label{tab:robust-radius}
  \resizebox{\textwidth}{!}{%
  \begin{tabular}{llrrrrrr}
    \toprule
    Dataset & Task & $L = \norm{\widehat w}_2$ & $u_{\max}$ & RobMetric$(0)$ & RobMetric$(u_{\max}/10)$ & RobMetric$(u_{\max})$ & $u_{\mathrm{half}}$ \\
    \midrule
    \texttt{breast\_cancer} & class.\ (Acc) & $397.92 \pm 15.54$ & $2.44\mathrm{e}{-3} \pm 2.88\mathrm{e}{-4}$ & $0.972 \pm 0.004$ & $0.925 \pm 0.012$ & $0.000 \pm 0.000$ & $1.74\mathrm{e}{-3} \pm 2.0\mathrm{e}{-4}$ \\
    \texttt{california\_h.\_2000} & reg.\ (MSE) & $193.42 \pm 3.38$ & $6.69\mathrm{e}{-4} \pm 2.4\mathrm{e}{-5}$ & $0.300 \pm 0.041$ & $4.137 \pm 0.116$ & $29.06 \pm 0.27$ & $2.33\mathrm{e}{-6} \pm 4.2\mathrm{e}{-7}$ \\
    \texttt{concrete} & reg.\ (MSE) & $1437.73 \pm 10.17$ & $3.29\mathrm{e}{-3} \pm 4.0\mathrm{e}{-5}$ & $18.45 \pm 1.75$ & $870.60 \pm 8.75$ & $7245.62 \pm 46.70$ & $3.16\mathrm{e}{-6} \pm 3.8\mathrm{e}{-7}$ \\
    \texttt{spambase\_full} & class.\ (Acc) & $759.30 \pm 4.12$ & $8.26\mathrm{e}{-5} \pm 6.7\mathrm{e}{-6}$ & $0.954 \pm 0.003$ & $0.846 \pm 0.016$ & $0.001 \pm 0.001$ & $3.35\mathrm{e}{-5} \pm 2.7\mathrm{e}{-6}$ \\
    \texttt{wine\_quality\_red} & reg.\ (MSE) & $118.11 \pm 0.95$ & $4.59\mathrm{e}{-3} \pm 7.3\mathrm{e}{-5}$ & $0.305 \pm 0.027$ & $9.109 \pm 0.130$ & $70.78 \pm 0.35$ & $6.16\mathrm{e}{-6} \pm 5.1\mathrm{e}{-7}$ \\
    \bottomrule
  \end{tabular}%
  }
\end{table}

\Cref{fig:robust-radius-panel} visualises the certificate curves on
the five datasets jointly. The classification panels show the
expected $\mathrm{RobAcc}(u)$ decay from the baseline test accuracy
at $u = 0$ down to $0$ at $u = u_{\max}$; the regression panels show
the $\mathrm{RobMSE}(u)$ growth from the baseline test MSE at
$u = 0$ up to the predicted ceiling at $u = u_{\max}$ (the
quadratic upper bound from \cref{eq:robust-mse}). The classification
$\mathrm{RobAcc}(u_{\max})$ is exactly zero by construction; the small
nonzero entries in \cref{tab:robust-radius} (of order $10^{-3}$, e.g.\
\texttt{spambase\_full}) are floating-point boundary artifacts,
corresponding to at most $k \le 2$ test points for which the strict
margin comparison $m_i > \norm{\widehat w}_2 \sqrt{u_{\max}}$ of
\cref{eq:robust-accuracy-curve} evaluates as true under floating
point, although in exact arithmetic no margin exceeds
$\norm{\widehat w}_2 \sqrt{u_{\max}}$.

\begin{figure}[t]
  \centering
  \includegraphics[width=\textwidth]{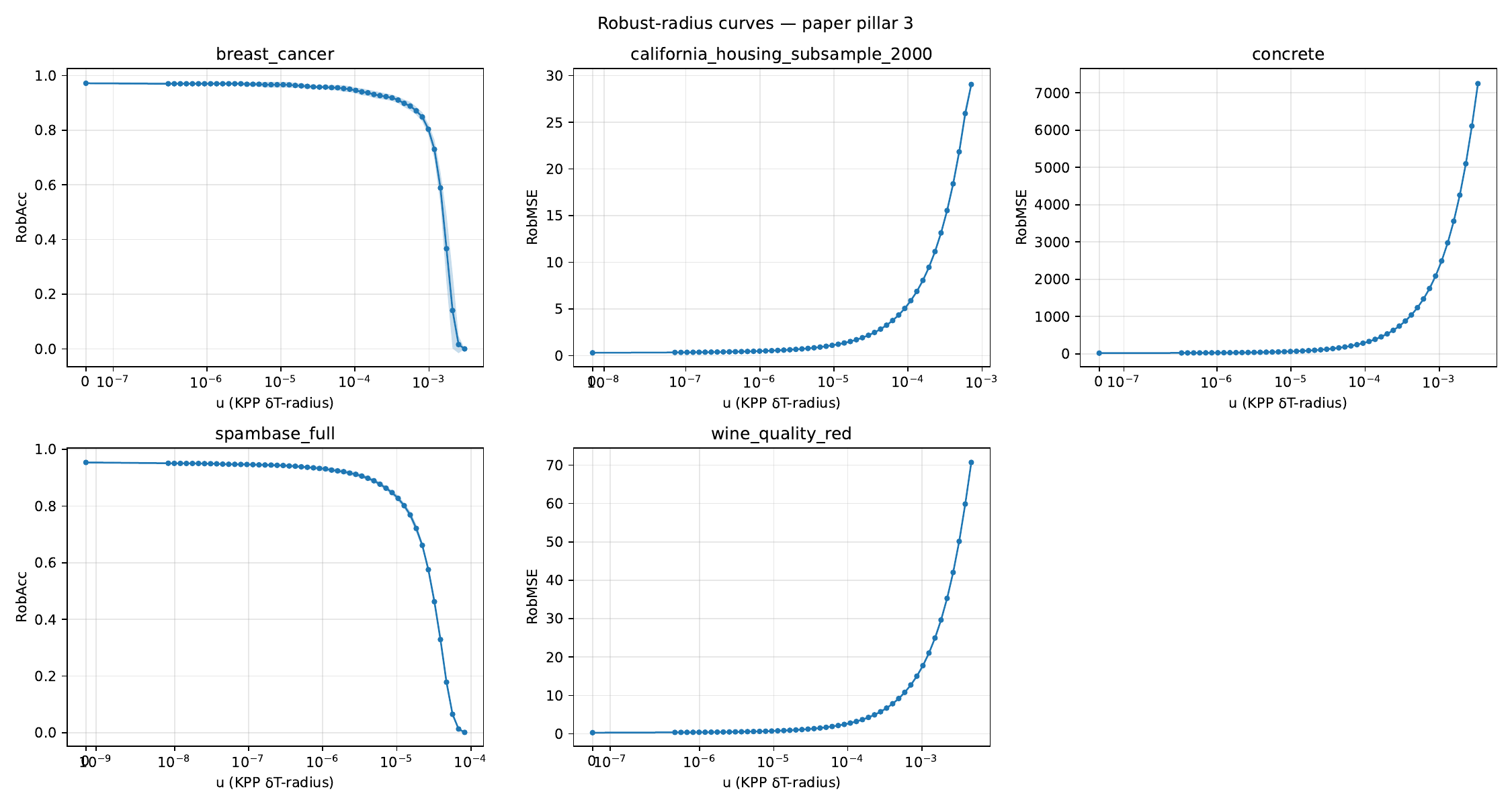}
  \caption{Robust-radius certificate curves on the five benchmark
    datasets, one panel per dataset arranged row-major in a
    $2 \times 3$ grid. Top row: \texttt{breast\_cancer}
    (classification, $\mathrm{RobAcc}(u)$ from
    \cref{eq:robust-accuracy-curve}),
    \texttt{california\_housing\_subsample\_2000} (regression,
    $\mathrm{RobMSE}(u)$ from \cref{eq:robust-mse}), and
    \texttt{concrete} (regression, $\mathrm{RobMSE}(u)$). Bottom row:
    \texttt{spambase\_full} (classification, $\mathrm{RobAcc}(u)$) and
    \texttt{wine\_quality\_red} (regression, $\mathrm{RobMSE}(u)$); the
    sixth cell is empty. Five-seed bands (shaded) and means (solid).
    Certificates are deterministic in the KPP path metric $\deltaT$ and
    should not be confused with robustness to arbitrary perturbations
    in the raw input metric.}
  \label{fig:robust-radius-panel}
\end{figure}

\subsection{Honest and cross-fit KPP on the sklearn Diabetes dataset}
\label{sec:diabetes-honest-crossfit}
The benchmark of \cref{sec:bench-five-datasets} reports KPP in the
full-sample regime only, its conditional diagnostics read in the
fixed-representation (PB-01) sense after the fitted representation is
frozen. The Conditional bounds RG-02 /
RG-03 apply equally well
under PB-02 (honest) and PB-03 (cross-fit) once the conditioning on
the KPP representation in the sense of \cref{sec:proof-boundary} is
enforced through a fold separation. Because the same labels build the representation and fit the final model, this full-sample setting is the label-dependent regime PB-04 of \cref{sec:proof-boundary}: the conditional bounds hold given the representation but do not by themselves constitute end-to-end guarantees for the full pipeline (cf.\ \cref{app:honesty-matters}). This subsection illustrates the
three regimes side by side on a single regression dataset, in order
to make visible (i) what is invariant across regimes, (ii) what
shifts, and (iii) the empirical cost of moving from PB-01 to PB-02
and PB-03. Diabetes is used as an illustration of the proof-boundary
conventions, not as a sixth row of the benchmark.

\paragraph{Setup.}
We use the \texttt{sklearn} Diabetes dataset ($n = 442$, $p = 10$,
regression). The forest is a $30$-tree Random Forest regressor with
\texttt{max\_depth = 6} and a fixed seed. The ridge penalty
$\lreg = 0.4329$ is selected once on the full sample by $5$-fold
cross-validation and reused across the three regimes; this is a
fair-comparison choice and is not a fully honest nested CV (under
PB-03, a strictly honest pipeline would re-select $\lreg$ on each
fold's fit data). Under PB-02 the partition fold and the fit fold
share each $50\%$ of the sample; under PB-03 we use $K = 5$
independent fold pairs, and the canonical training-set prediction
map is \texttt{predict\_train\_loo}, the leave-one-fold-out
predictor. We report single-seed values; this subsection
illustrates the proof-boundary conventions, not statistical
comparison (the multi-seed protocol is the role of
\cref{sec:bench-five-datasets}). The implementation is the
\texttt{honest\_pipeline} of the companion
\texttt{examples/honest\_demo.py} script.

\paragraph{Observed values.}
\Cref{tab:diabetes-honest} reports the train RMSE and four dashboard
observables for the three regimes. Capacity quantities
($\Tr(\Knode)/n$, $\deff(\lreg)/n$ of \cref{eq:effective-dimension},
$\mueig_{\min}(\Knode)$) are stable across regimes: the forest
configuration is the same and only the data splitting changes. The
post-fit norm $B = \norm{\widehat w}_2$ drops by roughly a third
between PB-01 and PB-02, consistent with the absence of label
leakage from the partition fold to the fit fold under honesty: the
ridge estimator no longer has to over-regularise against features
whose labels were used to grow the partitions. Under PB-03, the
five fold-specific norms $\norm{\widehat w_k}_2$ are tightly
concentrated (mean $662.7$, standard deviation $4.6$ across the five
folds), which is a stability diagnostic on the cross-fit estimator
that has no counterpart under PB-01 or PB-02.

\begin{table}[t]
  \centering
  \small
  \caption{Diabetes ($n=442$, $p=10$, regression): train RMSE and
    dashboard observables under the three proof-boundary regimes of
    \cref{sec:proof-boundary}. PB-01 fits forest and ridge on the
    full sample; PB-02 uses a $50\%/50\%$ partition/fit split;
    PB-03 uses $K=5$ independent fold pairs and reports the
    leave-one-fold-out predictor. The ridge penalty $\lreg = 0.4329$
    is shared across regimes for comparability and is not a fully
    honest nested CV. Single seed.}
  \label{tab:diabetes-honest}
  \begin{tabular}{lrrrrr}
    \toprule
    Regime & RMSE & $\Tr(\Knode)/n$ & $\deff(\lreg)/n$ & $\mueig_{\min}(\Knode)$ & $\norm{\widehat w}_2$ \\
    \midrule
    PB-01 (fixed)      & $45.40$ & $0.143$ & $0.073$ & $3.7\mathrm{e}{-5}$ & $727.4$ \\
    PB-02 (honest)     & $53.26$ & $0.195$ & $0.080$ & $4.8\mathrm{e}{-5}$ & $505.7$ \\
    PB-03 (cross-fit)  & $57.33$ & $0.161$ & $0.080$ & --- & $662.7 \pm 4.6$ \\
    \bottomrule
  \end{tabular}
\end{table}

\paragraph{Reading.}
The train RMSE increases monotonically from PB-01 to PB-02 to
PB-03 (leave-one-fold-out), from $45.4$ to $53.3$ to $57.3$. The three rows are not evaluated under a common protocol. PB-01
reports the resubstitution error of a pipeline whose forest and ridge
both saw every label, PB-02 evaluates on the full sample a ridge fit
on the disjoint fit fold only, a mixed in-sample and out-of-sample
evaluation, and PB-03 reports the leave-one-fold-out prediction map,
which is entirely out of fold. The monotone increase therefore
conflates two effects acting in the same direction, the removal of
in-sample optimism and the smaller label budget available to each
ridge fit, and should be read as an illustration of the three
pipeline conventions of \cref{sec:proof-boundary} rather than as a
like-for-like measurement of the cost of conditioning alone. The
forest-level variance-decrease diagnostic $IG_{\mathrm{var}}(T)$ of
\cref{cor:variance-oracle} drops by roughly $16\%$ between PB-01
and PB-02 ($3609 \to 3029$) for the same reason: the partition fold
under PB-02 contains half the labelled samples available under
PB-01. Under PB-03, the leave-one-fold-out predictor
\texttt{predict\_train\_loo} is the canonical training-sample
prediction map on which the Conditional bounds of
\cref{sec:proof-boundary} can be read directly; averaging the $K$
fold-specific predictors gives a distinct ensemble form whose
train RMSE on this dataset ($46.7$) is close to PB-01 but is not
the quantity bounded by RG-02 / RG-03 and should not be reported
as such. The asymmetry between regression and classification noted
in \cref{sec:radical-honesty} does not appear here because Diabetes
is a regression task; PB-02 and PB-03 admit analogous Conditional
bounds under any of the three regimes for both tasks, with the
same proof.

\section{Discussion}
\label{sec:discussion}

We close by enumerating the open problems and limitations that are
left explicit by the proof boundaries of the present paper, and by
naming the directions in which the KPP construction admits natural
extensions. None of what follows is claimed as work in progress;
each item is recorded so that the scope of the paper's results is
unambiguous.

\paragraph{Fast-rate excess-risk statements (FR-R-01, FR-C-01).}
The uniform Rademacher bounds of
\cref{prop:uniform-squared-loss-bound,thm:uniform-logistic-bound}
deliver $1/\sqrt n$ uniform risk control, conditionally on the KPP
representation and for uniformly bounded norm and response envelopes
$B, M$. Two fast-rate refinements are conjectured but not
proved in this paper. FR-R-01 is a regression fast-rate statement
under a Bernstein-type variance condition or a strong-convexity
condition specific to the KPP design, which would sharpen the
trace-based bound into a faster rate. FR-C-01 is the classification
analogue under \cref{ass:direct-score-margin}, combining the direct
score-margin tail with localised-complexity machinery. Both
statements require additional distributional assumptions and
additional technical apparatus, neither of which is specific to KPP;
they are stated here as open problems and are not promised as
follow-up work.

\paragraph{Open status of the pitfalls M-C-01 and O-C-02.}
The two obstructions named in \cref{sec:radical-honesty} are open
problems within the KPP representation, not artefacts of the present
analysis. For M-C-01 (\cref{app:counterexample-margin}), the
geometric-margin condition in the KPP path metric $\deltaT$ does not
imply a score-margin tail; closing this gap would require either a
distinct margin condition on the score or a structural assumption
on the trees that ties $\deltaT$-mass to $\abs{g_w}$-mass. For
O-C-02 (\cref{rem:raw-forest-sum-pitfall}), the raw forest sum of
per-tree refinement gains is not the refinement gain of a single
forest partition; constructing the correct common-refinement
oracle, together with its finite-sample concentration, is a
separate problem. We do not treat either as part of a planned
revision; they are recorded as the natural next questions a fully
KPP-specific theory would have to answer.

\paragraph{Proof boundaries and conditional guarantees.}
All probabilistic guarantees in this paper are stated conditionally
on the KPP representation, in the sense of \cref{sec:proof-boundary}:
the trees, the node-indexed feature map $\Phinode$, and the Gram
$\Knode$ are fixed before the final ridge estimator is analysed.
The honest and cross-fit constructions of
\cref{sec:proof-boundary,sec:diabetes-honest-crossfit} are the
mechanisms by which this conditioning is enforced in practice. A
fully unconditional analysis of a forest-and-ridge pipeline, in
which tree growth and ridge fitting are coupled through a single
training sample, is outside the present scope; it would require
either a stability analysis of the representation-building step or
a uniform complexity bound on the class of possible forests,
neither of which we attempt here (cf.\ \cref{app:honesty-matters}).
The robust-radius certificates of \cref{sec:robust-radius} are
deterministic in the KPP metric $\deltaT$ and do not certify the
same model against arbitrary perturbations in the raw input metric;
this is a property of the representation, not a defect to be
removed.

\paragraph{Extensions and natural directions.}
The construction extends in several directions that are not pursued
here. Multiclass classification with a one-vs-rest or multinomial
logistic surrogate would fit the same trace-based framework with
minor constant changes. Non-axis-aligned splits, oblique cuts, or
linear leaf models change the per-node geometry but leave the
squared-Euclidean path-isometry structure intact at the
representation level. Other regularised loss functions (Huber,
quantile) admit analogous uniform bounds under their standard
Lipschitz constants. The multi-ensemble setting --- in which $E$
ensembles, possibly differing in subsampling, feature selection, or
splitting protocol, are aggregated into a block-structured KPP Gram
--- preserves the squared-Euclidean path-isometric structure at the
representation level and is conjectured to admit analogous
trace-based Rademacher bounds under appropriate normalisation, with
the degenerate case of one tree per ensemble as a natural
intermediate. Each of these directions inherits the same
proof-boundary discipline as the present paper: they are stated as
extensions a future paper might pursue, not as claims of the
present one.

\bibliography{references}

\begin{thebibliography}{38}
\providecommand{\natexlab}[1]{#1}
\providecommand{\url}[1]{\texttt{#1}}
\expandafter\ifx\csname urlstyle\endcsname\relax
  \providecommand{\doi}[1]{doi: #1}\else
  \providecommand{\doi}{doi: \begingroup \urlstyle{rm}\Url}\fi

\bibitem[Agarwal et~al.(2022)Agarwal, Tan, Ronen, Singh, and
  Yu]{Agarwal2022HierarchicalShrinkage}
Abhineet Agarwal, Yan~Shuo Tan, Omer Ronen, Chandan Singh, and Bin Yu.
\newblock Hierarchical shrinkage: improving the accuracy and interpretability
  of tree-based methods.
\newblock In \emph{International Conference on Machine Learning}, 2022.

\bibitem[Agarwal et~al.(2025)Agarwal, Kenney, Tan, Tang, and
  Yu]{Agarwal2025MDIPlus}
Abhineet Agarwal, Ana~M. Kenney, Yan~Shuo Tan, Tiffany~M. Tang, and Bin Yu.
\newblock Integrating random forests and generalized linear models for improved
  accuracy and interpretability, 2025.
\newblock arXiv preprint; first version titled ``MDI+: A Flexible Random
  Forest-Based Feature Importance Framework''.

\bibitem[Andriushchenko \& Hein(2019)Andriushchenko and
  Hein]{AndriushchenkoHein2019ProvablyRobust}
Maksym Andriushchenko and Matthias Hein.
\newblock Provably robust boosted decision stumps and trees against adversarial
  attacks.
\newblock In \emph{Advances in Neural Information Processing Systems}, 2019.

\bibitem[Athey et~al.(2019)Athey, Tibshirani, and
  Wager]{AtheyTibshiraniWager2019GRF}
Susan Athey, Julie Tibshirani, and Stefan Wager.
\newblock Generalized random forests.
\newblock \emph{The Annals of Statistics}, 47\penalty0 (2):\penalty0
  1148--1178, 2019.

\bibitem[Bartlett \& Mendelson(2002)Bartlett and
  Mendelson]{BartlettMendelson2002Rademacher}
Peter~L. Bartlett and Shahar Mendelson.
\newblock Rademacher and gaussian complexities: Risk bounds and structural
  results.
\newblock \emph{Journal of Machine Learning Research}, 3:\penalty0 463--482,
  2002.

\bibitem[Bartlett et~al.(2006)Bartlett, Jordan, and
  McAuliffe]{BartlettJordanMcAuliffe2006Calibration}
Peter~L. Bartlett, Michael~I. Jordan, and Jon~D. McAuliffe.
\newblock Convexity, classification, and risk bounds.
\newblock \emph{Journal of the American Statistical Association}, 101\penalty0
  (473):\penalty0 138--156, 2006.

\bibitem[B{\'e}nard et~al.(2024)B{\'e}nard, N{\"a}f, and
  Josse]{BenardNafJosse2024MMD}
Cl{\'e}ment B{\'e}nard, Jeffrey N{\"a}f, and Julie Josse.
\newblock {MMD}-based variable importance for distributional random forests.
\newblock In \emph{International Conference on Artificial Intelligence and
  Statistics}, 2024.

\bibitem[Biau \& Scornet(2016)Biau and Scornet]{BiauScornet2016GuidedTour}
G{\'e}rard Biau and Erwan Scornet.
\newblock A random forest guided tour.
\newblock \emph{TEST}, 25\penalty0 (2):\penalty0 197--227, 2016.

\bibitem[Breiman(2001)]{Breiman2001RandomForests}
Leo Breiman.
\newblock Random forests.
\newblock \emph{Machine Learning}, 45\penalty0 (1):\penalty0 5--32, 2001.

\bibitem[Capitaine et~al.(2024)Capitaine, Bigot, Thiébaut, and
  Genuer]{Capitaine2024FrechetRF}
Louis Capitaine, Jérémie Bigot, Rodolphe Thiébaut, and Robin Genuer.
\newblock Fréchet random forests for metric space valued regression with
  non-{E}uclidean predictors.
\newblock \emph{Journal of Machine Learning Research}, 25:\penalty0 1--41,
  2024.

\bibitem[Chen et~al.(2019)Chen, Zhang, Boning, and
  Hsieh]{Chen2019RobustnessVerification}
Hongge Chen, Huan Zhang, Duane~S. Boning, and Cho-Jui Hsieh.
\newblock Robustness verification of tree-based models.
\newblock In \emph{Advances in Neural Information Processing Systems}, 2019.

\bibitem[Chen \& Guestrin(2016)Chen and Guestrin]{ChenGuestrin2016XGBoost}
Tianqi Chen and Carlos Guestrin.
\newblock Xgboost: A scalable tree boosting system.
\newblock In \emph{Proceedings of the 22nd ACM SIGKDD International Conference
  on Knowledge Discovery and Data Mining}, pp.\  785--794, 2016.

\bibitem[Davies \& Ghahramani(2014)Davies and
  Ghahramani]{DaviesGhahramani2014RandomPartitionKernels}
Alex Davies and Zoubin Ghahramani.
\newblock The random forest kernel and other kernels for big data from random
  partitions.
\newblock In \emph{Advances in Neural Information Processing Systems}, 2014.

\bibitem[Dem{\v{s}}ar(2006)]{demsar2006}
Janez Dem{\v{s}}ar.
\newblock Statistical comparisons of classifiers over multiple data sets.
\newblock \emph{Journal of Machine Learning Research}, 7:\penalty0 1--30, 2006.

\bibitem[Feng \& Zhou(2017)Feng and Zhou]{FengZhou2017AutoEncoderForest}
Ji~Feng and Zhi-Hua Zhou.
\newblock {AutoEncoder} by forest.
\newblock \emph{arXiv preprint arXiv:1709.09018}, 2017.

\bibitem[Friedman(2001)]{Friedman2001GBM}
Jerome~H. Friedman.
\newblock Greedy function approximation: A gradient boosting machine.
\newblock \emph{The Annals of Statistics}, 29\penalty0 (5):\penalty0
  1189--1232, 2001.

\bibitem[Friedman \& Popescu(2008)Friedman and
  Popescu]{FriedmanPopescu2008RuleEnsembles}
Jerome~H. Friedman and Bogdan~E. Popescu.
\newblock Predictive learning via rule ensembles.
\newblock \emph{Annals of Applied Statistics}, 2\penalty0 (3):\penalty0
  916--954, 2008.
\newblock \doi{10.1214/07-AOAS148}.

\bibitem[Geurts et~al.(2006)Geurts, Ernst, and
  Wehenkel]{GeurtsErnstWehenkel2006ExtraTrees}
Pierre Geurts, Damien Ernst, and Louis Wehenkel.
\newblock Extremely randomized trees.
\newblock \emph{Machine Learning}, 63\penalty0 (1):\penalty0 3--42, 2006.

\bibitem[Haddouchi \& Berrado(2024)Haddouchi and
  Berrado]{HaddouchiBerrado2024Survey}
Maissae Haddouchi and Abdelaziz Berrado.
\newblock A survey and taxonomy of methods interpreting random forest models.
\newblock \emph{arXiv preprint arXiv:2407.12759}, 2024.

\bibitem[Ke et~al.(2017)Ke, Meng, Finley, Wang, Chen, Ma, Ye, and
  Liu]{KeEtAl2017LightGBM}
Guolin Ke, Qi~Meng, Thomas Finley, Taifeng Wang, Wei Chen, Weidong Ma, Qiwei
  Ye, and Tie-Yan Liu.
\newblock Lightgbm: A highly efficient gradient boosting decision tree.
\newblock In \emph{Advances in Neural Information Processing Systems}, 2017.

\bibitem[Klusowski \& Tian(2024)Klusowski and
  Tian]{KlusowskiTian2024LargeScale}
Jason~M. Klusowski and Peter~M. Tian.
\newblock Large scale prediction with decision trees.
\newblock \emph{Journal of the American Statistical Association}, 119\penalty0
  (545):\penalty0 525--537, 2024.

\bibitem[Liang et~al.(2025)Liang, Rewolinski, Agarwal, Tang, and
  Yu]{Liang2025LMDIPlus}
Zachary Liang, Aaron Rewolinski, Abhineet Agarwal, Tiffany~M. Tang, and Bin Yu.
\newblock {LMDI+}: Local feature importances for tree-based models.
\newblock \emph{arXiv preprint arXiv:2506.08928}, 2025.

\bibitem[Linial et~al.(1995)Linial, London, and
  Rabinovich]{LinialLondonRabinovich1995Geometry}
Nathan Linial, Eran London, and Yuri Rabinovich.
\newblock The geometry of graphs and some of its algorithmic applications.
\newblock \emph{Combinatorica}, 15\penalty0 (2):\penalty0 215--245, 1995.

\bibitem[Lundberg \& Lee(2017)Lundberg and Lee]{LundbergLee2017SHAP}
Scott~M. Lundberg and Su-In Lee.
\newblock A unified approach to interpreting model predictions.
\newblock In \emph{Advances in Neural Information Processing Systems}, 2017.

\bibitem[Lundberg et~al.(2020)Lundberg, Erion, Chen, DeGrave, Prutkin, Nair,
  Katz, Himmelfarb, Bansal, and Lee]{LundbergErionChen2020TreeSHAP}
Scott~M. Lundberg, Gabriel~G. Erion, Hugh Chen, Alex DeGrave, Jordan~M.
  Prutkin, Bala Nair, Ronit Katz, Jonathan Himmelfarb, Nisha Bansal, and Su-In
  Lee.
\newblock From local explanations to global understanding with explainable ai
  for trees.
\newblock \emph{Nature Machine Intelligence}, 2\penalty0 (1):\penalty0 56--67,
  2020.

\bibitem[Mammen \& Tsybakov(1999)Mammen and
  Tsybakov]{MammenTsybakov1999SmoothDiscrimination}
Enno Mammen and Alexandre~B. Tsybakov.
\newblock Smooth discrimination analysis.
\newblock \emph{The Annals of Statistics}, 27\penalty0 (6):\penalty0
  1808--1829, 1999.

\bibitem[Mourtada et~al.(2020)Mourtada, Ga{\"\i}ffas, and
  Scornet]{MourtadaGaiffasScornet2020Mondrian}
Jaouad Mourtada, St{\'e}phane Ga{\"\i}ffas, and Erwan Scornet.
\newblock Minimax optimal rates for {Mondrian} trees and forests.
\newblock \emph{The Annals of Statistics}, 48\penalty0 (4):\penalty0
  2253--2276, 2020.
\newblock \doi{10.1214/19-AOS1886}.

\bibitem[Panda et~al.(2018)Panda, Shen, and Vogelstein]{Panda2018KMERF}
Sambit Panda, Cencheng Shen, and Joshua~T. Vogelstein.
\newblock Learning interpretable characteristic kernels via decision forests.
\newblock \emph{arXiv preprint arXiv:1812.00029}, 2018.

\bibitem[Prokhorenkova et~al.(2018)Prokhorenkova, Gusev, Vorobev, Dorogush, and
  Gulin]{Prokhorenkova2018CatBoost}
Liudmila Prokhorenkova, Gleb Gusev, Aleksandr Vorobev, Anna~Veronika Dorogush,
  and Andrey Gulin.
\newblock {CatBoost}: Unbiased boosting with categorical features.
\newblock In \emph{Advances in Neural Information Processing Systems}, 2018.

\bibitem[Raymaekers et~al.(2025)Raymaekers, Rousseeuw, Servotte, Verdonck, and
  Yao]{Raymaekers2025RaFFLE}
Jakob Raymaekers, Peter~J. Rousseeuw, Thomas Servotte, Tim Verdonck, and
  Ruicong Yao.
\newblock A powerful random forest featuring linear extensions ({RaFFLE}).
\newblock \emph{arXiv preprint arXiv:2502.10185}, 2025.

\bibitem[Rhodes et~al.(2023)Rhodes, Cutler, and Moon]{Rhodes2023RFGAP}
Jake~S. Rhodes, Adele Cutler, and Kevin~R. Moon.
\newblock Geometry- and accuracy-preserving random forest proximities.
\newblock \emph{arXiv preprint arXiv:2201.12682}, 2023.

\bibitem[Scornet(2016)]{Scornet2016RFKernel}
Erwan Scornet.
\newblock Random forests and kernel methods.
\newblock \emph{IEEE Transactions on Information Theory}, 62\penalty0
  (3):\penalty0 1485--1500, 2016.

\bibitem[Scornet \& Hooker(2025)Scornet and
  Hooker]{ScornetHooker2025TheoryReview}
Erwan Scornet and Giles Hooker.
\newblock Theory of random forests: A review.
\newblock \emph{HAL preprint hal-05006431}, 2025.

\bibitem[Shi \& Horvath(2006)Shi and Horvath]{ShiHorvath2006UnsupervisedRF}
Tao Shi and Steve Horvath.
\newblock Unsupervised learning with random forest predictors.
\newblock \emph{Journal of Computational and Graphical Statistics}, 15\penalty0
  (1):\penalty0 118--138, 2006.

\bibitem[Sutera et~al.(2021)Sutera, Louppe, Huynh-Thu, Wehenkel, and
  Geurts]{Sutera2021LocalMDI}
Antonio Sutera, Gilles Louppe, V{\^a}n~Anh Huynh-Thu, Louis Wehenkel, and
  Pierre Geurts.
\newblock From global to local {MDI} variable importances for random forests
  and when they are {Shapley} values.
\newblock In \emph{Advances in Neural Information Processing Systems}, 2021.

\bibitem[Tropp(2012)]{Tropp2012MatrixBernstein}
Joel~A. Tropp.
\newblock User-friendly tail bounds for sums of random matrices.
\newblock \emph{Foundations of Computational Mathematics}, 12\penalty0
  (4):\penalty0 389--434, 2012.

\bibitem[Vu et~al.(2025)Vu, Kapar, Wright, and
  Watson]{VuKaparWrightWatson2025RFAE}
Binh~Duc Vu, Jan Kapar, Marvin~N. Wright, and David~S. Watson.
\newblock Autoencoding random forests.
\newblock \emph{arXiv preprint arXiv:2505.21441}, 2025.

\bibitem[Wager \& Athey(2018)Wager and Athey]{WagerAthey2018HonestForests}
Stefan Wager and Susan Athey.
\newblock Estimation and inference of heterogeneous treatment effects using
  random forests.
\newblock \emph{Journal of the American Statistical Association}, 113\penalty0
  (523):\penalty0 1228--1242, 2018.

\end{thebibliography}
\bibliographystyle{tmlr}

\appendix

\section{Norm bound for the normalized embedding}
\label{app:norm-bound}

We prove \cref{prop:norm-bound}. Fix a tree $t$, and let
\[
  r_t = u_0, u_1, \dots, u_L = \ell_t(x)
\]
be the root-to-leaf path followed by $x$. The nonzero coordinates of
$\phit(x)$ are the root coordinate and the coordinates of the
non-root nodes on this path. Hence
\begin{align}
  \norm{\phit(x)}_2^2
  &= \frac{\nodew_t(r_t)}{2}
     + \sum_{j=1}^{L} \frac{\nodew_t(u_{j-1}) + \nodew_t(u_j)}{2} \\
  &= \nodew_t(r_t) + \sum_{j=1}^{L-1} \nodew_t(u_j),
\end{align}
where we used $\nodew_t(u_L) = 0$ because $u_L$ is a leaf and
\cref{ass:zero-leaf-weights} sets leaf weights to zero. Therefore
\begin{equation}
  \norm{\phit(x)}_2^2
  \;\le\; \sum_{v \in V_t} \nodew_t(v) \;=\; A_t.
\end{equation}
Summing over trees gives
\begin{equation}
  \norm{\phiT(x)}_2^2
  \;=\; \sum_{t \in T} \norm{\phit(x)}_2^2
  \;\le\; \sum_{t \in T} A_t
  \;=\; \ST.
\end{equation}
Dividing by $\ST$ yields $\norm{\Phinode(x)}_2 \le 1$.

\section{Nonzero leaf weights: what can and cannot be embedded}
\label{app:nonzero-leaf-weights}

The main text imposes \cref{ass:zero-leaf-weights}. This appendix
records the precise caveat behind that convention.

\paragraph{Why the raw node-path formula cannot include arbitrary leaf weights.}
With the path convention of \cref{lem:path-decomposition}, if two
inputs reach the same leaf $\ell$, then
$\sigma_t(x, x') = \set{\ell}$. Hence the raw formula
\begin{equation}
  d_t^{\mathrm{raw}}(x, x')
  \;:=\;
  \sum_{v \in \sigma_t(x, x')} \nodew_t(v)
\end{equation}
would give $d_t^{\mathrm{raw}}(x, x') = \nodew_t(\ell)$, even when
$x = x'$. Therefore no squared Euclidean embedding can satisfy
\begin{equation}
  \norm{\psi_t(x) - \psi_t(x')}_2^2
  \;=\;
  d_t^{\mathrm{raw}}(x, x')
\end{equation}
for all pairs if some leaf has positive weight: squared Euclidean
distances always vanish on the diagonal.

\paragraph{A zero-diagonal alternative with endpoint leaf contributions.}
If nonzero leaf weights are needed, the distance itself must be
adjusted so that points in the same leaf have distance zero. A
natural choice is the endpoint-leaf separation distance
\begin{equation}
  d_t^{\mathrm{sep}}(x, x')
  \;:=\;
  \sum_{v \in \sigma_t(x, x') \cap V_t^{\mathrm{int}}} \nodew_t(v)
     + \one\set{\ell_t(x) \ne \ell_t(x')}
       \bigl(\nodew_t(\ell_t(x)) + \nodew_t(\ell_t(x'))\bigr).
\end{equation}
This coincides with the usual internal-node path distance when leaf
weights are zero, and it also assigns zero distance to pairs falling
in the same leaf.

It has the following explicit embedding. For every non-root internal
node, keep
\begin{equation}
  \psi_{t, v}(x)
  \;=\;
  \sqrt{\frac{\nodew_t(\parnode(v)) + \nodew_t(v)}{2}}\,
    \one\set{x \in S_v},
  \qquad v \in V_t^{\mathrm{int}},\ v \ne r_t.
\end{equation}
For every leaf $\ell$, use
\begin{equation}
  \psi_{t, \ell}(x)
  \;=\;
  \sqrt{\frac{\nodew_t(\parnode(\ell))}{2} + \nodew_t(\ell)}\,
    \one\set{x \in S_\ell}.
\end{equation}
The root coordinate may again be constant. Then
\begin{equation}
  \norm{\psi_t(x) - \psi_t(x')}_2^2
  \;=\;
  d_t^{\mathrm{sep}}(x, x').
\end{equation}
Indeed, when the two leaves differ, the internal path nodes are
counted exactly by the edge-telescoping argument of
\cref{lem:node-path-edge-sum}, and the two endpoint leaf coordinates
contribute their own leaf weights. When the two inputs reach the
same leaf, all indicator coordinates agree, so both sides are zero.

This construction is not used in the main text. The canonical KPP
core in \cref{sec:kpp-construction} keeps the cleaner zero-leaf
convention of \cref{ass:zero-leaf-weights}.

\section{Ridge primal-dual identity}
\label{app:ridge-primal-dual}

We supply the algebraic identity used to recast the fitted KPP-Ridge
predictor as the kernel-form $\widehat y = \Knode(\Knode + \lreg I_n)^{-1} y$
of \cref{prop:ridge-bias-variance}.

Let $\Phi \in \R^{n \times m}$ denote the design matrix with rows
$\Phinode(x_i)^\top$, so that $\Knode = \Phi \Phi^\top$ on the
training inputs. The empirical KPP-Ridge objective at regularisation
$\lreg > 0$\footnote{The regularisation convention differs between the
two tasks: this regression objective is \emph{unaveraged} (the loss is a
sum over the $n$ training samples), whereas the classification objective
of \cref{eq:kpp-logistic-objective} averages the loss by $1/n$. The
penalty $\tfrac{\lreg}{2}\norm{w}_2^2$ has the same form in both, so the
same $\lreg$ acts at an effective strength of $\lreg/n$ relative to the
averaged loss in regression versus $\lreg$ in classification. Each task's
bounds are stated consistently with its own convention, but the raw
$\lreg$ is not directly comparable across tasks.} is
\begin{equation}
  \widehat w_{\lreg}
  \;:=\;
  \arg\min_{w \in \R^m}
    \tfrac{1}{2}\norm{y - \Phi w}_2^2
    + \tfrac{\lreg}{2}\,\norm{w}_2^2,
  \label{eq:kpp-ridge-objective}
\end{equation}
with first-order condition
\begin{equation}
  (\Phi^\top \Phi + \lreg I_m)\, \widehat w_{\lreg}
  \;=\;
  \Phi^\top y.
  \label{eq:ridge-normal-equation}
\end{equation}
The matrix identity
\begin{equation}
  (\Phi^\top \Phi + \lreg I_m)^{-1} \Phi^\top
  \;=\;
  \Phi^\top (\Phi \Phi^\top + \lreg I_n)^{-1}
  \label{eq:ridge-push-through}
\end{equation}
follows by multiplying both sides on the left by
$(\Phi^\top \Phi + \lreg I_m)$ and on the right by
$(\Phi \Phi^\top + \lreg I_n)$ and checking
$\Phi^\top \Phi \Phi^\top + \lreg \Phi^\top = \Phi^\top \Phi \Phi^\top + \lreg \Phi^\top$.
Substituting into the normal equation gives
\begin{equation}
  \widehat w_{\lreg}
  \;=\;
  \Phi^\top (\Knode + \lreg I_n)^{-1} y,
\end{equation}
and multiplying by $\Phi$ on the left yields the fitted-value
identity
\begin{equation}
  \widehat y
  \;=\;
  \Phi\, \widehat w_{\lreg}
  \;=\;
  \Knode (\Knode + \lreg I_n)^{-1} y
  \;=\;
  S_{\lreg}\, y,
\end{equation}
where $S_{\lreg}$ is the ridge smoother of
\cref{prop:ridge-bias-variance,eq:effective-dimension}. This is the
expression used in the bias--variance decomposition of
\cref{eq:ridge-bias-variance}.

\section{Honesty matters: conditional bounds versus full-pipeline guarantees}
\label{app:honesty-matters}

The Rademacher and ridge-identity statements of
\cref{sec:rademacher-regression,sec:rademacher-classification} are
all conditional on the KPP representation $\Phinode$ in the sense of
\cref{sec:proof-boundary}. This appendix records what they do, and do
not, prove about the full pipeline
\begin{equation*}
  (X_i, Y_i)_{i=1}^{n}
  \;\longrightarrow\;
  (T, \nodew, \Phinode)
  \;\longrightarrow\;
  \widehat w
\end{equation*}
of \cref{eq:full-kpp-pipeline}.

If the same labels are used both to choose the KPP representation
(splits, node weights) and to fit the final ERM weights $\widehat w$,
then the hypothesis class itself is label-dependent: the class
$\Fcal_B = \set{x \mapsto \inner{w}{\Phinode(x)} : \norm{w}_2 \le B}$
of \cref{eq:fcal-B-linear} depends on the same labels that the
empirical Rademacher complexity averages over. This does not make
the algorithm invalid; it changes the proof problem. The uniform
squared-loss bound \cref{prop:uniform-squared-loss-bound} and the
uniform logistic bound \cref{thm:uniform-logistic-bound} then become
statements about the fitted representation, not about the full data-
to-predictor map.

A clean end-to-end guarantee requires one of the three regimes of
\cref{sec:proof-boundary}:
\begin{itemize}
  \item \textbf{Fixed representation.} The forest, node weights, and
    features are regarded as fixed before the final learner is
    analysed; the conditional bound is directly a bound on the final
    predictor.
  \item \textbf{Honest representation.} The sample is split into a
    partition fold (which builds the trees and node weights) and a
    fit fold (which constructs the KPP design and fits the ERM); the
    conditional bound is then conditional on the partition fold and
    averages over the labels in the fit fold only.
  \item \textbf{Cross-fit representation.} For each fold $q$, a
    representation $\Phinode^{(-q)}$ is learned without using labels
    from fold $q$, and points in fold $q$ are embedded with that
    representation; the conditional bound is applied fold-wise.
\end{itemize}
Any KPP statistical theorem in this paper carries one of these three
clauses. An alternative route --- proving a full-pipeline guarantee
through stability of the representation-building step or a uniform
complexity bound on the class of possible forests --- is consistent
with this framework but is not pursued here.

\section{Variance gain as a node-weight choice}
\label{app:variance-gain-choice}

For regression, a natural internal-node weight is the variance
decrease
\begin{equation}
  \nodew_t(v) \;:=\; \Delta Q(v),
  \qquad v \in V_t^{\mathrm{int}},
  \label{eq:variance-gain-node-weight}
\end{equation}
with leaf weights set to zero under \cref{ass:zero-leaf-weights}.
With this choice, the KPP path distance
$\dT(x, x') = \sum_t \sum_{v \in \sigma_t(x, x')} \Delta Q(v)$
aligns geometrically with the variance oracle of
\cref{cor:variance-oracle}: pairs separated by high-gain splits sit
at large path distance, while pairs falling in the same well-pooled
leaf sit at distance zero.

Other nonnegative weight families are admissible --- impurity decrease
in the Gini sense, raw conditional variance, uniform weights, or
hand-crafted relevance scores --- but the oracle interpretation
changes accordingly. The unified geometry and isometry results of
\cref{sec:kpp-construction} hold for any nonnegative weight family
satisfying \cref{ass:zero-leaf-weights}.

Empirically, this choice is competitive and robust rather than optimal.
\Cref{tab:node-weight} compares the three weightings head to head on the
40-dataset suite of \cref{sec:bench-demsar}, run on a shared forest with
the per-dataset $\lreg$ held fixed so that the arms differ only in the
node weights; the default arm (variance / impurity decrease) reproduces
the benchmark KPP estimator bit for bit. The choice has a measurable
effect --- the arms are far from interchangeable --- yet no weighting wins
universally: against the default, raw conditional variance and uniform
weights each improve on it on $15$ of the $40$ datasets, tie on $5$, and
are worse on $20$, and on $15$ datasets neither alternative beats the
default. The better arm is dataset-dependent, while the default is beaten
only on a minority and holds the plurality of the head-to-head
comparisons. This is the behaviour expected of a reasoned design choice:
the variance-gain weighting of \cref{eq:variance-gain-node-weight} is
robust within the admissible family, not a proven per-dataset optimum, and
we do not present it as optimal. The comparison is an internal ablation of
the node weights alone, on a different footing from the cross-method
ranking of \cref{sec:bench-demsar} and from the leaf-only representation
ablation reported in \cref{sec:bench-five-datasets}.

\begin{table}[t]
  \centering
  \small
  \caption{Node-weight ablation on the 40-dataset suite of
    \cref{sec:bench-demsar} ($20$ regression $+$ $20$ classification, five
    seeds). The three admissible weightings of this appendix --- variance
    / impurity decrease (the default), raw conditional variance, and
    uniform --- are run on a shared forest with the per-dataset $\lreg$
    held fixed, so the arms differ only in the node weights; the default
    arm reproduces the benchmark KPP estimator bit for bit. Each
    alternative is compared head to head against the default per dataset
    on the seed-averaged primary metric (RMSE for regression, error rate
    $1 - \mathrm{accuracy}$ for classification, both lower-is-better);
    ``Wins'' is from the alternative's perspective (strictly better than
    the default), a difference at or below the numerical-precision floor
    ($\lvert \Delta \rvert \le 10^{-12}$) counting as a tie. On $15$ of the
    $40$ datasets neither alternative beats the default.}
  \label{tab:node-weight}
  \begin{tabular}{lrrr}
    \toprule
    Alternative vs.\ default & Wins & Ties & Losses \\
    \midrule
    raw conditional variance & $15$ & $5$ & $20$ \\
    uniform                  & $15$ & $5$ & $20$ \\
    \bottomrule
  \end{tabular}
\end{table}

\section{Crude logistic calibration: proof of the pointwise inequality}
\label{app:logistic-calibration-proof}

We prove the pointwise inequality
\cref{eq:logistic-pointwise-calibration} used in
\cref{sec:rademacher-classification}: for every $u \in \R$ and
$y \in \set{-1, +1}$,
\begin{equation*}
  \one\set{y u \le 0}
  \;\le\;
  \frac{\ell_{\log}(u, y)}{\log 2},
  \qquad \ell_{\log}(u, y) = \log(1 + e^{-y u}).
\end{equation*}

If $y u > 0$, the left-hand side is zero and the right-hand side is
nonnegative, so the inequality is trivial. If $y u \le 0$, then
$e^{-y u} \ge 1$, hence $\ell_{\log}(u, y) = \log(1 + e^{-y u}) \ge \log 2$,
so the right-hand side is at least $1$, while the left-hand side
equals $1$. Taking expectations of both sides over
$(X, Y) \sim P$ at the score $u = g_w(X)$ recovers
\cref{eq:logistic-calibration}.

\section{Entropy telescoping: direct derivation}
\label{app:entropy-telescoping-proof}

We give the direct derivation of \cref{cor:entropy-oracle}, as an
alternative to the task-neutral specialisation of
\cref{lem:partition-telescoping} given in the main text.

For every internal node $v$, multiplying the definition of
$\Delta h(v)$ (see \cref{eq:delta-h}) by $p(v)$ and using
$p(v) \pi_L(v) = p(v_L)$, $p(v) \pi_R(v) = p(v_R)$ gives
\begin{equation}
  p(v)\, h(\eta(v))
  \;=\;
  p(v_L)\, h(\eta(v_L)) + p(v_R)\, h(\eta(v_R)) + p(v)\, \Delta h(v).
\end{equation}
Summing over internal nodes telescopes: every non-root node appears
exactly once as a child of its parent. The surviving terms are the
root entropy $h(\eta(r_t))$, the leaf entropies
$\sum_{\ell \in V_t^{\mathrm{leaf}}} p(\ell)\, h(\eta(\ell)) = L_{\log}^{\mathrm{leaf}}(t)$,
and the accumulated information gain
$\sum_{v \in V_t^{\mathrm{int}}} p(v)\, \Delta h(v) = IG_h(t)$,
yielding \cref{eq:entropy-oracle-tree}.

\section{Counterexample template: geometric margin does not imply score margin}
\label{app:counterexample-margin}

The implication
\begin{equation}
  \Prob\bigl(\deltaT(X, \mathcal B) \le r\bigr) \le C r^{\nu}
  \quad \Longrightarrow \quad
  \Prob\bigl(\abs{g(X)} \le \gamma\bigr) \le C' \gamma^{\alpha}
  \label{eq:false-margin-implication}
\end{equation}
is false in general, where $\mathcal B = \set{x : g(x) = 0}$ is
the score's zero set and the left-hand side is a geometric-margin
mass condition in the KPP path metric $\deltaT$ of
\cref{sec:setup-representation}. Here $g$ denotes a general
$\sqrt{\deltaT}$-Lipschitz score on the embedding metric. A Lipschitz score can be nearly
flat and close to zero on a large set far from its zero boundary.
Lipschitzness controls how quickly the score can change as one moves
in $\deltaT$, but it does not force the score to grow away from the
boundary. A lower-margin condition, or the direct score-margin
condition of \cref{ass:direct-score-margin}, is therefore required
for any fast-rate argument in \cref{sec:robust-radius-classification}.

\paragraph{A one-dimensional example.}
Take $X \sim \mathrm{Unif}([0, 1])$, the Bayes boundary
$\mathcal B = \set{1/2}$, and the super-flat score
$g(x) = \sign(x - 1/2)\, e^{-1/\abs{x - 1/2}}$ with $g(1/2) := 0$.
Writing $t = \abs{x - 1/2}$, the derivative
$\abs{g'(x)} = t^{-2} e^{-1/t}$ is bounded on $[0, 1]$ and tends to $0$
as $t \to 0$, so $g$ is Lipschitz in the embedding metric
$\sqrt{\deltaT}$ whenever $\sqrt{\deltaT}$ dominates the Euclidean
metric; the geometric-margin mass condition
$\Prob(\abs{X - 1/2} \le r) = 2 r$ holds with $\nu = 1$ (equivalently $\nu = 1/2$ in $\deltaT$ via
\cref{rem:deltaT-half-metric}); the conclusion fails for every
$\alpha$ regardless of $\nu$.
Simultaneously $\abs{g(x)} \le \gamma$ if and only if
$\abs{x - 1/2} \le 1 / \ln(1/\gamma)$, so for small $\gamma$,
\begin{equation*}
  \Prob\bigl(\abs{g(X)} \le \gamma\bigr) = \frac{2}{\ln(1/\gamma)},
\end{equation*}
which decays more slowly than any power $\gamma^{\alpha}$ with
$\alpha > 0$, so no score-margin tail of the form $C' \gamma^{\alpha}$
can hold. Every KPP score $g_w$ is $\norm{w}_2$-Lipschitz in
$\sqrt{\deltaT}$, hence lies in this ambient class; the witness $g$ is
not a realisable finite KPP score $g_w$ but exhibits the obstruction
within the ambient Lipschitz class. Whether a realisable $g_w$ can
violate the score-margin tail uniformly across refining forests is the
open problem recorded in \cref{sec:radical-honesty}.
Geometric margin therefore does not imply a score-margin
tail, and the direct score-margin condition of
\cref{ass:direct-score-margin} is what closes the gap. This is the
obstruction anchored as the M-C-01 pitfall in the radical-honesty
discussion of \cref{sec:radical-honesty}.

\section*{Supplementary Material}
\phantomsection
\label{suppl:bench-detailed}

\subsection*{Per-dataset benchmark tables}

This appendix supplies the per-dataset detailed bench tables
deferred from \cref{sec:bench-five-datasets}. Each table reports the
primary metric (RMSE for regression, test error rate
$1 - \mathrm{accuracy}$ for classification, both lower-is-better),
the mean fit time in seconds, and the number of seeds for the six
baselines and the KPP estimator. Five seeds per (dataset, method),
mean $\pm$ standard deviation. The aggregation is produced by
\texttt{examples/aggregate\_benchmarks.py} from the per-dataset CSVs.

\begin{table}[t]
  \centering
  \small
  \caption{\texttt{breast\_cancer} --- classification, $n_{\mathrm{train}} = 455$,
    $n_{\mathrm{test}} = 114$, $p = 30$. Test error rate (lower is better).}
  \label{tab:bench-breast-cancer}
  \begin{tabular}{lrrr}
    \toprule
    Method & Error rate (mean $\pm$ std) & Fit time (s, mean) & Seeds \\
    \midrule
    KPP              & $0.0281 \pm 0.0039$ & $0.74$ & $5$ \\
    KPP\_leaf-only   & $0.0351 \pm 0.0107$ & $-$    & $5$ \\
    RandomForest     & $0.0351 \pm 0.0164$ & $0.15$ & $5$ \\
    GradientBoosting & $0.0351 \pm 0.0196$ & $0.82$ & $5$ \\
    XGBoost          & $0.0316 \pm 0.0078$ & $0.07$ & $5$ \\
    LightGBM         & $0.0351 \pm 0.0164$ & $0.24$ & $5$ \\
    CatBoost         & $0.0246 \pm 0.0130$ & $0.43$ & $5$ \\
    LogReg\_raw      & $0.0193 \pm 0.0073$ & $0.00$ & $5$ \\
    \bottomrule
  \end{tabular}
\end{table}

\begin{table}[t]
  \centering
  \small
  \caption{\texttt{california\_housing\_subsample\_2000} --- regression,
    $n_{\mathrm{train}} = 1\,600$, $n_{\mathrm{test}} = 400$, $p = 8$.
    RMSE (lower is better).}
  \label{tab:bench-california}
  \begin{tabular}{lrrr}
    \toprule
    Method & RMSE (mean $\pm$ std) & Fit time (s, mean) & Seeds \\
    \midrule
    KPP              & $0.5469 \pm 0.0413$ & $4.74$ & $5$ \\
    KPP\_leaf-only   & $0.5709 \pm 0.0389$ & $-$    & $5$ \\
    RandomForest     & $0.5710 \pm 0.0387$ & $0.39$ & $5$ \\
    GradientBoosting & $0.5402 \pm 0.0435$ & $0.68$ & $5$ \\
    XGBoost          & $0.5582 \pm 0.0472$ & $0.18$ & $5$ \\
    LightGBM         & $0.5287 \pm 0.0421$ & $0.33$ & $5$ \\
    CatBoost         & $0.5183 \pm 0.0379$ & $0.14$ & $5$ \\
    Ridge\_raw       & $0.6793 \pm 0.0424$ & $0.00$ & $5$ \\
    \bottomrule
  \end{tabular}
\end{table}

\begin{table}[t]
  \centering
  \small
  \caption{\texttt{concrete} --- regression, $n_{\mathrm{train}} = 824$,
    $n_{\mathrm{test}} = 206$, $p = 8$. RMSE (lower is better).}
  \label{tab:bench-concrete}
  \begin{tabular}{lrrr}
    \toprule
    Method & RMSE (mean $\pm$ std) & Fit time (s, mean) & Seeds \\
    \midrule
    KPP              & $4.3024 \pm 0.2299$ & $1.16$ & $5$ \\
    KPP\_leaf-only   & $4.7186 \pm 0.2705$ & $-$    & $5$ \\
    RandomForest     & $4.8994 \pm 0.3106$ & $0.19$ & $5$ \\
    GradientBoosting & $4.7117 \pm 0.5452$ & $0.24$ & $5$ \\
    XGBoost          & $4.4086 \pm 0.3166$ & $0.18$ & $5$ \\
    LightGBM         & $4.2472 \pm 0.2464$ & $0.37$ & $5$ \\
    CatBoost         & $4.3316 \pm 0.3106$ & $0.12$ & $5$ \\
    Ridge\_raw       & $10.7906 \pm 0.7432$ & $0.00$ & $5$ \\
    \bottomrule
  \end{tabular}
\end{table}

\begin{table}[t]
  \centering
  \small
  \caption{\texttt{spambase\_full} --- classification,
    $n_{\mathrm{train}} = 3\,680$, $n_{\mathrm{test}} = 921$, $p = 57$.
    Test error rate (lower is better).}
  \label{tab:bench-spambase}
  \begin{tabular}{lrrr}
    \toprule
    Method & Error rate (mean $\pm$ std) & Fit time (s, mean) & Seeds \\
    \midrule
    KPP              & $0.0452 \pm 0.0033$ & $65.84$ & $5$ \\
    KPP\_leaf-only   & $0.0443 \pm 0.0031$ & $-$     & $5$ \\
    RandomForest     & $0.0465 \pm 0.0039$ & $0.34$  & $5$ \\
    GradientBoosting & $0.0510 \pm 0.0072$ & $2.77$  & $5$ \\
    XGBoost          & $0.0502 \pm 0.0058$ & $0.25$  & $5$ \\
    LightGBM         & $0.0456 \pm 0.0051$ & $0.55$  & $5$ \\
    CatBoost         & $0.0456 \pm 0.0077$ & $0.68$  & $5$ \\
    LogReg\_raw      & $0.0743 \pm 0.0100$ & $0.01$  & $5$ \\
    \bottomrule
  \end{tabular}
\end{table}

\begin{table}[t]
  \centering
  \small
  \caption{\texttt{wine\_quality\_red} --- regression,
    $n_{\mathrm{train}} = 1\,279$, $n_{\mathrm{test}} = 320$, $p = 11$.
    RMSE (lower is better).}
  \label{tab:bench-wine}
  \begin{tabular}{lrrr}
    \toprule
    Method & RMSE (mean $\pm$ std) & Fit time (s, mean) & Seeds \\
    \midrule
    KPP              & $0.5506 \pm 0.0272$ & $2.85$ & $5$ \\
    KPP\_leaf-only   & $0.5408 \pm 0.0253$ & $-$    & $5$ \\
    RandomForest     & $0.5577 \pm 0.0279$ & $0.32$ & $5$ \\
    GradientBoosting & $0.6016 \pm 0.0265$ & $0.48$ & $5$ \\
    XGBoost          & $0.5823 \pm 0.0299$ & $0.17$ & $5$ \\
    LightGBM         & $0.5854 \pm 0.0249$ & $0.33$ & $5$ \\
    CatBoost         & $0.5811 \pm 0.0174$ & $0.12$ & $5$ \\
    Ridge\_raw       & $0.6291 \pm 0.0205$ & $0.00$ & $5$ \\
    \bottomrule
  \end{tabular}
\end{table}

\end{document}